\documentclass{scrartcl}
\pdfoutput=1
\usepackage[utf8]{inputenc} 
\usepackage[T1]{fontenc}    
\usepackage{amsfonts}       
\usepackage{amsmath}
\usepackage{amssymb}
\usepackage{mathtools}
\usepackage{amsthm}
\usepackage{bm}
\usepackage{color}
\usepackage{soul}           
\usepackage{url}
\usepackage{booktabs}       
\usepackage{nicefrac}       
\usepackage{enumerate}
\usepackage{graphicx}
\DeclareGraphicsExtensions{.pdf,.png,.jpg}
\graphicspath{{figures/}}
\usepackage{algorithm}
\usepackage{algpseudocode}
\usepackage{comment}
\usepackage{hyperref}
\usepackage{centernot}     
\usepackage{natbib}
\usepackage{multirow}
\usepackage[titletoc,title]{appendix}

\usepackage{fullpage}
\newcommand{\R}{\mathbb{R}}

\newcommand{\BigO}[1]{\ensuremath{\mathcal{O}(#1)}}                             
\newcommand{\BigOm}[1]{\ensuremath{\Omega(#1)}}                                 

\newcommand{\vect}[1]{\ensuremath{\mathbf{#1}}}                                 
\newcommand{\vectsym}[1]{\ensuremath{\boldsymbol{#1}}}                          
\newcommand{\mat}[1]{\ensuremath{\mathbf{\MakeUppercase{#1}}}}                  
\newcommand{\Exp}[2]{\ensuremath{\mathbb{E}_{#1}\left[#2\right]}}                
\newcommand{\Cov}[2]{\ensuremath{\mathrm{Cov}_{#1}\left[#2\right]}}              
\newcommand{\Ind}[1]{\ensuremath{\mathbf{1}\left[#1\right]}}                     

\newcommand{\Norm}[1]{\ensuremath{\lVert #1 \rVert}}                  
\newcommand{\NormI}[1]{\ensuremath{\lVert #1 \rVert}_1}               
\newcommand{\NormII}[1]{\ensuremath{\lVert #1 \rVert}_2}              
\newcommand{\NormInfty}[1]{\ensuremath{\lVert #1 \rVert_{\infty}}}    

\newcommand{\matrx}[1]{\begin{pmatrix}#1\end{pmatrix}}                           

\newcommand{\InNorm}[1]{{\left\vert\kern-0.2ex\left\vert\kern-0.2ex\left\vert #1 
    \right\vert\kern-0.2ex\right\vert\kern-0.2ex\right\vert}}                    

\newcommand{\InNormII}[1]{{\left\vert\kern-0.2ex\left\vert\kern-0.2ex\left\vert #1 
    \right\vert\kern-0.2ex\right\vert\kern-0.2ex\right\vert}_2}                    

\newcommand{\InNormInfty}[1]{{\left\vert\kern-0.2ex\left\vert\kern-0.2ex\left\vert #1 
    \right\vert\kern-0.2ex\right\vert\kern-0.2ex\right\vert}_{\infty}}           

\newcommand{\Abs}[1]{\ensuremath{\lvert #1 \rvert}}                              
\newcommand{\Prob}[1]{\ensuremath{\mathrm{Pr} \{ #1 \}}}               
\newcommand{\iid}{i.i.d.~}                                                        
\DeclarePairedDelimiterX{\Inner}[2]{\langle}{\rangle}{#1, #2}                    

\newcommand{\Land}{\wedge}                                                       
\newcommand{\defeq}{\overset{\mathrm{def}}{=}}                                   

\DeclareMathOperator*{\union}{\cup}
\DeclareMathOperator*{\intersection}{\cap}

\DeclareMathOperator*{\argmin}{argmin}

\newtheorem{definition}{Definition}
\newtheorem{proposition}{Proposition}
\newtheorem{assumption}{Assumption}
\newtheorem{lemma}{Lemma}
\newtheorem{theorem}{Theorem}
\newtheorem{remark}{Remark}

\newcommand{\mA}{\mat{A}}
\newcommand{\mB}{\mat{B}}

\newcommand{\mI}{\mat{I}}

\newcommand{\mX}{\mat{X}}

\newcommand{\mSig}{\mat{\Sigma}}
\newcommand{\mOmg}{\mat{\Omega}}

\newcommand{\vb}{\vect{b}}

\newcommand{\ve}{\vect{e}}

\newcommand{\vr}{\vect{r}}

\newcommand{\vv}{\vect{v}}
\newcommand{\vw}{\vect{w}}
\newcommand{\vx}{\vect{x}}
\newcommand{\vy}{\vect{y}}
\newcommand{\vz}{\vect{z}}
\newcommand{\vbeta}{\vectsym{\beta}}
\newcommand{\veps}{\vectsym{\varepsilon}}

\newcommand{\vth}{\vectsym{\theta}}
\newcommand{\vomg}{\vectsym{\omega}}

\allowdisplaybreaks 

\newcommand{\Set}[1]{\{#1\}}    
\newcommand{\Vs}{\mathsf{V}}  
\newcommand{\Es}{\mathsf{E}}  
\newcommand{\G}{\mathsf{G}}   
\newcommand{\Gh}{\widehat{\mathsf{G}}}   
\newcommand{\Ws}{\mathsf{W}}  
\newcommand{\Wsh}{\widehat{\mathsf{W}}}  
\newcommand{\Ss}{\mathsf{S}}  
\newcommand{\Sh}{\widehat{\mathsf{S}}}  
\newcommand{\Pf}{\mathcal{P}}  
\newcommand{\Par}[2]{\pi_{#2}(#1)}    
\newcommand{\hPar}[2]{\widehat{\pi}_{#2}(#1)}    
\newcommand{\Chi}[2]{\phi_{#2}(#1)}   

\newcommand{\mi}{\mathsf{-i}}         
\newcommand{\corr}{\mathrm{corr}}       
\newcommand{\Sp}{\mathcal{S}}           
\newcommand{\mhOmg}{\widehat{\mOmg}}    
\newcommand{\mtOmg}{\mOmg^*}            
\newcommand{\mhB}{\widehat{\mB}}        
\newcommand{\vhth}{\widehat{\vth}}      
\newcommand{\bomg}{\bar{\omega}}  
\newcommand{\homg}{\widehat{\omega}}    
\newcommand{\vbomg}{\bar{\vectsym{\omega}}}  
\newcommand{\eigmin}{\lambda_{\mathrm{min}}}   
\newcommand{\eigmax}{\lambda_{\mathrm{max}}}   
\newcommand{\smin}{s_{\mathrm{min}}}   
\newcommand{\smax}{s_{\mathrm{max}}}   
\newcommand{\cmax}{C_{\mathrm{max}}}  
\newcommand{\cmin}{C_{\mathrm{min}}}  
\newcommand{\wtlmax}{\widetilde{w}_{\mathrm{max}}}  
\newcommand{\mbrOmg}{\bar{\mOmg}}    
\newcommand{\Ts}{\mathcal{T}}        
\newcommand{\Si}{\Ss_i}              
\newcommand{\dels}{\vectsym{\Delta}_{\Ss}}   
\newcommand{\wtl}{\widetilde{w}}   

\begin{document} 

\title{Learning Identifiable Gaussian Bayesian Networks in Polynomial Time and Sample Complexity}

\author{Asish Ghoshal and Jean Honorio\\
Department of Computer Science\\
Purdue University\\
West Lafayette, IN - 47906\\
\{aghoshal, jhonorio\}@purdue.edu}

\date{}

\maketitle

\begin{abstract}
Learning the directed acyclic graph (DAG) structure of a Bayesian network from observational data is a notoriously difficult problem for which many hardness results are known. In this paper we propose a provably polynomial-time algorithm for learning sparse Gaussian Bayesian networks with equal noise variance --- a class of Bayesian networks for which the DAG structure can be uniquely identified from observational data --- under high-dimensional settings. We show that $\BigO{k^4 \log p}$ number of samples suffices for our method to recover the true DAG structure with high probability, where $p$ is the number of variables and $k$ is the maximum Markov blanket size. We obtain our theoretical guarantees under a condition called Restricted Strong Adjacency Faithfulness, which is strictly weaker than strong faithfulness --- a condition that other methods based on conditional independence testing need for their success. The sample complexity of our method matches the information-theoretic limits in terms of the dependence on $p$. We show that our method out-performs existing state-of-the-art methods for learning Gaussian Bayesian networks in terms of recovering the true DAG structure while being comparable in speed to heuristic methods.
\end{abstract}

\section{Introduction}
\paragraph{Motivation.} The problem of learning the directed acyclic graph (DAG) structure of Bayesian networks (BNs) in general, and Gaussian Bayesian networks (GBNs) --- or equivalently linear Gaussian structural equation models (SEMs) --- in particular, from observational data has a long history in the statistics and machine learning community. This is, in part, motivated by the desire to uncover causal relationships between entities in domains as diverse as finance, genetics, medicine, neuroscience and artificial intelligence, to name a few. Although in general, the DAG structure of a GBN or linear Gaussian SEM cannot be uniquely identified from purely observational data (i.e., multiple structures can encode the same conditional independence relationships present in the observed data set), under certain restrictions on the generative model, the DAG structure can be uniquely determined. Furthermore, the problem of learning the structure of BNs exactly is known to be NP-complete even when the number of parents of a node is at most $q$, for $q > 1$, \cite{chickering1996learning}. It is also known that approximating the log-likelihood to a constant factor, even when the model class is restricted to polytrees with at-most two parents per node, is NP-hard \cite{dasgupta1999learning}. 
\vspace*{-0.1in}
\paragraph{Contribution.} In this paper we develop a polynomial time algorithm for learning a subclass of BNs exactly: sparse GBNs with equal noise variance. Our algorithm involves estimating a $p$-dimensional inverse covariance matrix and solving $2(p - 1)$ at-most-$k$-dimensional ordinary least squares problems, where $p$ is the number of nodes and $k$ is the maximum Markov blanket size of a variable. We show that 
$\BigO{(\nicefrac{k^4}{\alpha^2}) \log (\nicefrac{p}{\delta})}$ samples suffice for our algorithm to recover the true DAG structure and to approximate the parameters to at most $\alpha$ additive error, with probability at least $1 - \delta$, for some $\delta > 0$. The sample complexity of $\BigO{k^4 \log p}$ is close to the information-theoretic limit of $\BigO{k \log p}$ for learning sparse GBNs as obtained by \cite{ghoshal2016information}.
The main assumption under which we obtain our theoretical guarantees is a condition that we refer to as the $\alpha$-\emph{restricted strong adjacency faithfulness} (RSAF). We show that RSAF is a strictly weaker condition than \emph{strong faithfulness}, which methods based on independence testing require for their success. Through simulation experiments we demonstrate that our method recovers the true DAG structure
perfectly. 
\section{Related Work}
In the this section, we first discuss some identifiability results for GBNs known in the literature and then survey relevant algorithms for learning GBNs and Gaussian SEMs.

\cite{peters_causal_2014} proved identifiability of distributions
drawn from a restricted SEM with additive noise, where in the restricted SEM the functions are assumed to be non-linear and thrice continuously differentiable. It is also known that SEMs with linear functions and non-Gaussian noise are identifiable \cite{Shimizu2006}. Indentifiability of the DAG structure for the linear function and Gaussian noise case was proved by \cite{Peters2014} when noise variables are assumed to have equal variance.

Algorithms for learning BNs typically fall into two distinct categories, namely: 
independence test based methods and score based methods.
This dichotomy also extends to the Gaussian case. Score based methods assign a score to a candidate DAG structure based on how well it explains the observed data, and then attempt to
find the highest scoring structure. Popular examples for the Gaussian distribution are the log-likelihood based BIC and AIC scores and the $\ell_0$-penalized log-likelihood score by  \cite{van_de_geer_l0-penalized_2013}. However, given that the number of DAGs and sparse DAGs is exponential in the number of variables \cite{Robinson1977, ghoshal2016information}, searching for the highest scoring DAG in the combinatorial space of all DAGs is prohibitive for all but a few number of variables. \cite{aragam2015concave} propose a score-based method, based on concave penalization of a reparameterized negative log-likelihood function, which can learn a GBN over 1000 variables in an hour. However, the resulting optimization problem is neither convex --- therefore is not guaranteed to find a globally optimal solution --- nor solvable in polynomial time. In light of these shortcomings, approximation algorithms have been proposed for learning BNs which can be used to learn GBNs in conjunction with a suitable score function; notable methods are Greedy Equivalence Search (GES) proposed by \cite{chickering_optimal_2003} and an LP-relaxation based method proposed by \cite{jaakkola_learning_2010}. 

Among independence test based methods for learning GBNs, \cite{kalisch_estimating_2007} extended the PC algorithm to learn the Markov equivalence class of GBNs from observational data. The computational complexity of the PC algorithm is bounded by $\BigO{p^k}$ with high probability, and is only efficient for  learning very sparse DAGs. For the non-linear Gaussian SEM case, \cite{peters_causal_2014} developed
a two-stage algorithm called RESIT, which works by first learning the causal ordering of the variables and then performing regressions to learn the DAG structure. However, as we show in Proposition \ref{prop:resit_prop} (see Appendix \ref{app:appendix_discussion}), RESIT does not work for the linear Gaussian case. Moreover, Peters et al. proved the correctness of RESIT only in the population setting. Lastly, \cite{park2015learning} developed an algorithm, which is similar in spirit to our algorithm, for efficiently learning Poisson Bayesian networks. They exploit a property specific to the Poisson distribution called overdispersion to learn the causal ordering of variables.

Finally, the max-min hill climbing (MMHC) algorithm by \cite{tsamardinos2006max} is a state-of-the-art hybrid algorithm for BNs that combines ideas from constraint-based and score-based learning. While MMHC works well in practice, it is inherently a heuristic algorithm and is not guaranteed to recover the true DAG structure even when it is uniquely identifiable.
\section{Preliminaries}
In this section, we formalize the problem of learning Gaussian Bayesian networks from observational data. First,
we introduce some notations and definitions. We denote the set $\Set{1, \ldots, p}$ by $[p]$. Vectors and matrices are
denoted by lowercase and uppercase bold faced letters respectively. Random variables (including random vectors)
are denoted by italicized uppercase letters. Let $s_r, s_c \subseteq [p]$ be any two non-empty index sets. Then for any
matrix $\mA \in \R^{p \times p}$, we denote the matrix formed by selecting from $\mA$ the rows and columns in $s_r$ and $s_c$ respectively by: $\mA_{s_r, s_c} \in \R^{\Abs{s_r} \times \Abs{s_c}}$. 
With a slight abuse of notation, we will allow the index sets $s_r$ and $s_c$ to be a single index, e.g., $i$, and we will denote the index set of all row (or columns) by $*$. Thus, $\mA_{*,i}$ and $\mA_{i,*}$ denote the $i$-th column and row of $\mA$ respectively. For any vector $\vv \in \R^p$, we will denote its support set by: $\Sp(\vv) = \Set{i \in [p] | \Abs{v_i} > 0 }$.
Vector $\ell_p$-norms are denoted by $\Norm{\cdot}_p$. For matrices, $\Norm{\cdot}_p$ denotes the induced (or operator) $\ell_p$-norm
and $\Abs{\cdot}_p$ denotes the elementwise $\ell_p$-norm, i.e., $\Abs{\mA}_p \defeq (\sum_{i,j} \Abs{A_{i,j}}^p)^{\nicefrac{1}{p}}$.
Finally, we denote the set $[p] \setminus \Set{i}$ by $\mi$.

Let $\G = (\Vs, \Es)$ be a directed
acyclic graph (DAG) where the vertex set $\Vs = [p]$ and $\Es$ is the set of directed edges, where $(i,j) \in \Es$ implies the edge $i \leftarrow j$. 
We denote by $\Par{i}{\G}$ and $\Chi{i}{\G}$ the parent set and the set of children of the $i$-th node respectively,
in the graph $\G$; and drop the subscript $\G$ when the intended graph is clear from context. A vertex $i \in [p]$
is a \emph{terminal vertex} in $\G$ if $\Chi{i}{\G} = \varnothing$. For each $i \in [p]$ we have
a random variable $X_i \in \R$, $X = (X_1, \ldots, X_p)$ is the $p$-dimensional vector of random variables, 
and $\vx = (x_1, \ldots, x_p)$ is a joint assignment to $X$. Without loss of generality, we assume that
$\Exp{}{X_i} = 0,\, \forall i \in [p]$.
Every DAG $\G = (\Vs, \Es)$ defines a set of 
topological orderings $\Ts_{\G}$ over $[p]$ that are compatible with the DAG $\G$, i.e., $\Ts_{\G} = \Set{\tau \in \mathrm{S}_p | \tau(j) < \tau(i) \text{ if } (i,j) \in \Es}$, where $\mathrm{S}_p$ is the set of all possible permutations of $[p]$. 

A Gaussian Bayesian network (GBN) is a tuple $(\G, \Pf(\Ws, \Ss))$, where $\G=(\Vs, \Es)$ is a DAG structure, $\Ws = \Set{w_{i,j} \in \R  \:|\: (i, j) \in \Es \Land \Abs{w_{i,j}} > 0}$ is the set of edge weights, $\Ss = \Set{\sigma_i^2 \in \R_+}_{i=1}^p$ is the set of noise variances, and $\Pf$ is a multivariate Gaussian distribution over $X = (X_1,\ldots,X_p)$ that is \emph{Markov} with respect to the DAG $\G$ and is parameterized by $\Ws$ and $\Ss$. In other words, $\Pf = \mathcal{N}(\vx; \vect{0}, \mSig)$, factorizes as follows:
\begin{gather}
\Pf(\vx; \Ws, \Ss) = \prod_{i=1}^p \Pf_i(x_i ; \vw_i, \vx_{\Par{i}{}}, \sigma_i^2), \label{eq:joint_dist} \\
\Pf_i(x_i; \vw_i, \vx_{\Par{i}{}}, \sigma_i^2) = \mathcal{N}(x_i; \vw_i^T \vx_{\Par{i}{}}, \sigma_i^2 ), \label{eq:cond_dist}
\end{gather}
where $\vw_i \in \R^{\Abs{\Par{i}{}}} \defeq (w_{i,j})_{j \in \Par{i}{}}$ is the weight vector for the $i$-th node,
$\vect{0}$ is a vector of zeros of appropriate dimension (in this case $p$),
$\vx_{\Par{i}{}} = \Set{x_j | j \in \Par{i}{}}$, $\mSig$ is the covariance matrix for $X$, 
and $\Pf_i$ is the conditional distribution of $X_i$ given its parents --- which is also Gaussian.

We will also extensively use an alternative, but equivalent, view of a GBN:
the \emph{linear structural equation model} (SEM). 
Let $\mB = (w_{i,j} \Ind{(i,j) \in \Es})_{(i,j) \in [p] \times [p]}$ be the 
matrix of weights created from the set of edge weights $\Ws$.
A GBN $(\G, \Pf(\Ws, \Ss))$ corresponds to a SEM where each variable $X_i$ can be written as follows:
\begin{align}
X_i = \sum_{j \in \Par{i}{}} B_{i,j} X_{j} + N_i,\, \forall i \in [p] \label{eq:struct_eq}
\end{align}
with $N_i \sim \mathcal{N}(0, \sigma_i^2)$ (for all $i \in [p]$) being independent noise variables and $\Abs{B_{i,j}} > 0$ for all $j \in \Par{i}{}$. The joint distribution of $X$ as given by the SEM corresponds to the distribution $\Pf$ in \eqref{eq:joint_dist} and the graph associated with the SEM, where we have a directed edge $(i,j)$ if $j \in \Par{i}{}$, corresponds to the DAG $\G$. Denoting $N = (N_1, \ldots, N_p)$ as the noise vector, \eqref{eq:struct_eq} can be rewritten in vector form as: $X = \mB X + N$. 

Given a GBN $(\G, \Pf(\Ws, \Ss))$, with $\mB$ being the weight matrix corresponding to $\Ws$, we denote the \emph{effective influence} between two nodes $i, j \in [p]$  
\begin{align}
\wtl_{i,j} \defeq B_{i,j} + B_{j,i} - \mB_{*,i}^T \mB_{*,j} 
\label{eq:effective_influence}
\end{align}
The effective influence $\wtl_{i,j}$ between two nodes $i$ and $j$ is zero if: (a) $i$ and $j$ do not have an edge between them and do not have common children, or (b) $i$ and $j$ have an edge between them but the dot product between the weights to the children ($\mB_{*,i}^T \mB_{*,j}$) exactly
equals the edge weight between $i$ and $j$ ($B_{i,j} + B_{j,i}$). The effective influence determines the Markov blanket of each node, i.e., $\forall i \in [p]$, the Markov blanket  is given as: $\Ss_i = \Set{j \:|\: j \in \mi \Land \wtl_{i,j} \neq 0}$. Furthermore, a node is conditionally independent of all other nodes not in its Markov blanket, i.e., $\Prob{X_i | X_{\mi}} = \Prob{X_i | X_{\Ss_i}}$.
Next, we present a few definitions that will be useful later. 

\begin{definition}[Causal Minimality \cite{Zhang2008}] A distribution $\Pf$ is \emph{causal minimal} with respect to a DAG structure $\G$ if it is not Markov with respect to a proper subgraph of $\G$.
\end{definition}

\begin{definition}[Faithfulness \cite{spirtes2000causation}] Given a GBN $(\G, \Pf)$, $\Pf$ is faithful to the DAG $\G = (\Vs, \Es)$
if for any $i, j \in \Vs$ and any $\Vs' \subseteq \Vs \setminus \Set{i,j}$:
\begin{align*}
i \text{ d-separated from } j | \Vs' \iff \corr(X_i, X_j | X_{\Vs'}) = 0,
\end{align*}
where $\corr(X_i, X_j | X_{\Vs'})$ is the partial correlation between $X_i$ and $X_j$ given $X_{\Vs'}$.
\end{definition}

\begin{definition}[Strong Faithfulness \cite{zhang2002strong}] Given a GBN $(\G, \Pf)$ the multivariate Gaussian distribution
$\Pf$ is $\lambda$-strongly faithful to the DAG $\G$, for some $\lambda \in (0, 1)$, if
\begin{gather*}
\min \Set{\Abs{\corr(X_i, X_j | X_{\Vs'})}, 
	(i \text{ is not d-separated from } j | \Vs') | \\
	 \forall i,j \in [p] \Land \forall \Vs' \subseteq \Vs \setminus \Set{i,j}} \geq \lambda.
\end{gather*}
\end{definition}
Thus, strong faithfulness is a stronger version of the faithfulness assumption that requires that for all triples $(X_i, X_j, X_{\Vs'})$ such that $i \text{ is not d-separated from } j$ given $\Vs'$, the partial correlation $\corr(X_i, X_j | X_{\Vs'})$ is bounded away from $0$. It has been shown that while the set of distributions $\Pf$ that are Markov to a DAG $\G$ but not faithful to it have Lebesgue measure zero, the set of distributions $\Pf$ that are not strongly faithful to $\G$ have nonzero Lebesgue measure, and in fact can be quite large \cite{uhler_geometry_2013}.

The problem of learning a GBN from observational data corresponds to recovering the  DAG structure $\G$ and parameters $\Ws$ from a matrix $\mX \in \R^{n \times p}$ of $n$ \iid samples drawn from $\Pf(\Ws, \Ss)$. In this paper we consider the problem of learning GBNs over $p$ variables where the size of the Markov blanket of a node is at most $k$. This is in general not possible without making
additional assumptions on the GBN $(\G, \Pf(\Ws, \Ss))$ and the distribution $\Pf$ as we describe next.

\subsection{Assumptions}
In this section, we enumerate our technical assumptions.
\begin{assumption}[Causal Minimality] 
\label{ass:causal_minimality}
Let $(\G, \Pf(\Ws, \Ss))$ be a GBN, then $\forall w_{i,j} \in \Ws$, $\Abs{w_{i,j}} > 0$.
\end{assumption}
The above assumption ensures that all edge weights are strictly nonzero, as a result of which we have that each variable $X_i$ is a non-constant function of its parents $X_{\Par{i}{}}$. Given Assumption \ref{ass:causal_minimality}, 
the distribution $\Pf$ is causal minimal with respect to $\G$ \cite{peters_causal_2014} 
and therefore identifiable \cite{Peters2014} under equal noise variances, i.e., 
$\sigma_1 = \ldots = \sigma_p = \sigma$. Throughout the rest of the paper, we will denote
such Bayesian networks by $(\G, \Pf(\Ws, \sigma^2))$. 

\begin{assumption}[Restricted Strong Adjacency Faithfulness]
\label{ass:faithfulness}
Let $(\G, \Pf(\Ws, \sigma^2))$ be a GBN with $\G = (\Vs, \Es)$. 
For every $\tau \in \Ts_{\G}$,
consider the sequence of graphs $\G[m,\tau] = (\Vs[m, \tau], \Es[m, \tau])$ 
indexed by $(m, \tau)$, where $\G[m, \tau]$ is
the induced subgraph of $\G$ over the first $m$ vertices in the topological ordering $\tau$,
i.e., $\Vs[m, \tau] \defeq \Set{i \in [p] \:|\: \tau(i) \leq m}$ and 
$\Es[m, \tau] \defeq \Set{(i,j) \in \Es \:|\: i \in \Vs[m, \tau] \Land j \in \Vs[m, \tau]}$.
The multivariate Gaussian distribution $\Pf$ is restricted
 $\alpha$-strongly adjacency faithful to $\G$, if the following hold: 
\begin{align*}
&(i) ~ \min \Set{\Abs{w_{i,j}} \:|\: (i,j) \in \Es} > 3 \alpha, \\
&(ii) ~ \Abs{\wtl_{i,j}} > \frac{3 \alpha}{\kappa(\alpha)}  ,\: \\
&\qquad	\forall  i \in \Vs[m, \tau] \Land 
		j \in \Ss_i[m, \tau] \Land m \in [p] \Land \tau \in \Ts_{\G},
\end{align*}
where $\alpha > 0$ is a constant, $\wtl_{i,j}$ is the effective influence 
between $i$ and $j$ in the induced subgraph $\G[m, \tau]$ as defined in \eqref{eq:effective_influence}, and
$\Ss_i[m, \tau]$ denotes the Markov blanket of node $i$ in $\G[m, \tau]$. 
The constant $\kappa(\alpha) = 1 - \nicefrac{2}{(1 + 9 \Abs{\Chi{i}{\G[m, \tau]}} \alpha^2)}$ 
if $i$ is a non-terminal vertex in $\G[m, \tau]$, where
$ \Abs{\Chi{i}{\G[m, \tau]}}$ is the number of children of $i$ in $\G[m, \tau]$, 
and $\kappa(\alpha) = 1$ if $i$ is a terminal vertex.
\end{assumption}

Simply stated, the RSAF assumption requires that the absolute value of the edge weights
are at least $3 \alpha$ and the absolute value of the effective influence between two nodes,
 whenever it is non-zero, is at least 
$3 \alpha $ for terminal nodes and $\nicefrac{3\alpha}{\kappa(\alpha)}$ for non-terminal nodes.
Moreover, the above should hold not only for the original DAG, 
but also for each DAG obtained by sequentially removing terminal vertices.
Note that in the regime $\alpha \in (0, \nicefrac{1}{3  \sqrt{\Abs{\Chi{i}{\G[m, \tau]}}}})$, which is the case
when the estimation error $\alpha$ is low, then the condition on $\wtl_{i,j}$ is satisfied trivially.
As we will show later, the Assumption \ref{ass:faithfulness} is equivalent to the following: 
\begin{gather*}
\min\Set{\Abs{\corr(X_i, X_j | X_{\Vs[m,\tau] \setminus \Set{i,j}} )}\; |\; i \in \Vs[m,\tau]\\
 \Land j \in \Ss_i[m,\tau]
 \Land m \in [p] \Land \tau \in \Ts_{\G}}
 \geq \alpha',
\end{gather*}
for some constant $\alpha'$.

The constant $\alpha$ is related to the statistical error when using a finite number of samples and decays as $\BigOm{\sqrt{\nicefrac{\log k}{n}}}$. This implies that
as the number of samples $n \rightarrow \infty$, the ``minimum signal'' requirement for the edge weights and effective influence goes to $0$. In the RSAF assumption we require that not only the edge weights but also the \emph{effective influence} between a node and another node in its Markov blanket, to be bounded away from zero. This is to ensure that the inverse covariance matrix, or precision matrix, correctly recovers the undirected skeleton of $\G$. 
An illustration of a GBN that violates the RSAF assumption is shown in Figure \ref{fig:faithfulness}.
\begin{figure}[htbp]
\begin{center}
\includegraphics[width=0.5\linewidth]{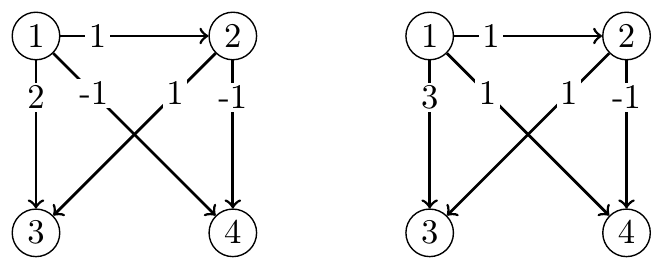}
\end{center}
\vspace*{-0.2in}
\caption{ (Left) GBN, with noise variance set to 1, in which the joint distribution is 
\emph{not} Restricted Strong Adjacency Faithful (RSAF) to the DAG structure because 
in the induced subgraph over nodes $(1, 2, 4)$, we have that the effective influence 
between $1$ and $2$, $\wtl_{1,2} = 0$ (nodes $1$ and $2$ have an edge between them, 
have a common child $4$, and $B_{4,1} B_{4,2} = B_{1,2}$). (Right) GBN in which the joint distribution is RSAF to the DAG structure. Note that the GBN in the left is also not faithful to the DAG structure since $\corr(X_1, X_2 | X_4) = 0$.
\label{fig:faithfulness}}
\end{figure}

Our final assumption requires that the Gaussian distribution $\Pf$ is non-singular.
\begin{assumption}[Non-singularity] 
\label{ass:nonsingular}
Given a GBN $(\G, \Pf(\Ws, \sigma^2))$, the multivariate Gaussian distribution $\Pf$ is non-singular if the covariance matrix is positive definite, i.e., $\eigmin(\mSig) > 0$, and $\eigmax(\mSig) < \infty$, where $\eigmin(.)$ (respectively $\eigmax(.)$) denotes the minimum (respectively maximum) eigenvalue.
\end{assumption}

At this point, it is worthwhile to compare our assumptions with those made by other methods for learning GBNs. Methods based on conditional independence (CI) tests, e.g., the PC algorithm for learning the equivalence class of GBNs developed by \cite{kalisch_estimating_2007}, require strong faithfulness. While strong faithfulness requires that for a node pair $(i,j)$ that are adjacent in the DAG, the partial correlation $\corr(X_i,X_j | X_{\Ss})$ is bounded away from zero for
all sets $\Ss \in \Set{\Ss \subseteq [p] \setminus \Set{i,j}}$, RSAF only requires non-zero partial correlations with respect to a subset of sets in $\Set{\Ss \subseteq [p] \setminus \Set{i,j}}$. Thus, RSAF is strictly weaker than strong faithfulness. Moreover, the number of non-zero partial correlations needed by RSAF
is also strictly a subset of those needed by the faithfulness condition. But, to tolerate statistical errors, we additionally need that the non-zero partial correlations to be bounded away from 0. An example of a GBN which is RSAF but neither faithful, nor strongly faithful, nor adjacency faithful 
(see \cite{uhler_geometry_2013} for a definition) is shown in Figure \ref{fig:rsaf_example}.
\begin{figure}[htbp]
\begin{center}
\includegraphics[width=0.2\linewidth]{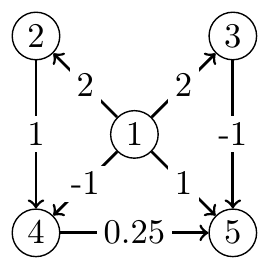}
\end{center}
\vspace*{-0.2in}
\caption{A GBN, with noise variance set to 1, that is RSAF to the DAG structure with $\alpha = 0.25$, but is neither faithful, nor strongly faithful, nor adjacency faithful to the DAG structure.
Consider the pair $(4,5)$. RSAF only requires that $\corr(X_4, X_5 | \varnothing) \neq 0$ which is true in this case ($\corr(X_4, X_5 | \varnothing) = 0.25$). However, we have that $\corr(X_4, X_5 | X_2, X_3)  = 0$ even though $(2,3)$ do not d-separate $4$ and $5$. Other violations of faithfulness include $\corr(X_1, X_4 | \varnothing) = 0$ and $\corr(X_1, X_5 | \varnothing) = 0$. Therefore, a CI test based method will fail to recover the true structure.%
\label{fig:rsaf_example}}
\end{figure}

We conclude this section with one last remark. At first glance, 
it might appear that the assumption of equal variance together with our assumptions implies a simple causal ordering of variables in which the marginal variance of the variables increases monotonically with the causal ordering. However, this is not the case. For instance, in the GBN shown in Figure \ref{fig:rsaf_example} the marginal variance of the causally ordered nodes $(1, 2, 3, 4, 5)$ is $(1,  2,  2,  2,  1.625)$.
\section{Results}
We start off this section by characterizing the covariance and precision matrix for a GBN $(\G, \Pf(\Ws, \sigma^2))$.
Let $\mB$ be the weight matrix corresponding to the edge weights $\Ws$, from \eqref{eq:struct_eq}
it follows that the covariance and precision matrix are, respectively:
\begin{align}
\mSig &= \sigma^2 (\mI - \mB)^{-1} (\mI - \mB)^{-T} 
\label{eq:covariance_mat} \\
\mOmg &= \frac{1}{\sigma^2} (\mI - \mB)^{T} (\mI - \mB), \label{eq:precision_mat}
\end{align}
where $\mI$ is the $p \times p$ identity matrix.
The following technical lemma characterizes the precision matrix $\mOmg$ and the conditional mean of the $i$-th random variable, given all other variables, in terms of the weight matrix $\mB$.
\begin{lemma}
\label{lemma:theta}
Let $(\G, \Pf(\Ws, \sigma^2))$ be a GBN, $\mB$ be the weight matrix corresponding to $\Ws$ and $\mOmg = (\Omega_{i,j})$ be the
inverse covariance matrix over $X$. 
For all $j \neq i$, we have that: $\Omega_{i,j} = (\nicefrac{1}{\sigma^2})(B_{i,j} + B_{j,i} - \mB_{*i}^T\mB_{*j})$,
$\Omega_{i,i} = (\nicefrac{1}{\sigma^2})(1 + \mB_{*i}^T\mB_{*i})$
and  $\Exp{}{X_i | (X_{\mi} = \vx_{\mi})} = \vth_i^T \vx_{\mi}$, where 
\begin{align*}
\theta_{ij} = - \frac{\Omega_{i,j}}{\Omega_{i,i}} = 
	\frac{B_{i,j} + B_{j,i} - \mB_{*i}^T\mB_{*j}}{1 + \mB_{*i}^T\mB_{*i}}.
\end{align*}
\end{lemma}
Please see Appendix \ref{app:detailed_proofs} for detailed proofs.
\begin{remark}
Since the elements of the inverse covariance matrix are related to the partial correlations as follows:
$\corr(X_i, X_j | X_{\Vs \setminus \Set{i,j}}) = -\nicefrac{\Omega_{i,j}}{\sqrt{\Omega_{i,i} \Omega_{j,j}}}$.
From Lemma \ref{lemma:theta} we have that, $\wtl_{i,j} \geq c \alpha$ (Assumption \ref{ass:faithfulness}) implies that
$\Abs{\corr(X_i, X_j | X_{\Vs \setminus \Set{i,j}})} \geq \nicefrac{c \alpha}{\sqrt{\Omega_{i,i} \Omega_{j,j}}} > 0$.
\end{remark}
The following lemma describes a key property of terminal vertices.
\begin{lemma}
\label{lemma:terminal_vertex}
Let $(\G, \Pf(\Ws, \sigma^2))$ be a GBN  with $\mOmg$ being the inverse covariance matrix over $X$
and $\vth_i$ being the regression coefficients as given in Lemma \ref{lemma:theta}.
Under Assumption \ref{ass:causal_minimality}, we have that
\begin{align*}
\text{$i$ is a terminal vertex in $\G$} \iff \theta_{ij} = - \sigma^2 \Omega_{i,j},\, \forall j \in \mi.
\end{align*}
\end{lemma}

Lemma \ref{lemma:terminal_vertex} states that, in the population setting, one can identify the terminal vertex, and therefore the causal ordering, 
just by assuming causal minimality (Assumption \ref{ass:causal_minimality}). However, to identify terminal vertices from a finite number of samples, one needs additional assumptions. We use Lemma \ref{lemma:terminal_vertex} to develop our algorithm for learning GBNs which, at a high level, works as follows. Given data $\mX$ drawn from a GBN, we first estimate the inverse covariance matrix $\mhOmg$.
Then we perform a series of ordinary least squares (OLS) regressions to compute the estimators $\vhth_i \, \forall i \in [p]$. 
We then identify terminal vertices using the property described in Lemma \ref{lemma:terminal_vertex} and remove the corresponding variables (columns) from $\mX$. We repeat the process of identifying and removing terminal vertices and obtain the causal ordering of vertices. Then, we perform a final set of OLS regressions to learn the structure and parameters of the DAG. 

The two main operations performed by our algorithm are: (a) estimating the inverse covariance matrix, and (b) estimating the regression coefficients $\vth_i$. In the next few subsections we discuss these two steps in more detail and obtain theoretical guarantees for our algorithm.
\subsection{Inverse covariance matrix estimation}
The first part of our algorithm requires an estimate $\mhOmg$ of the true 
inverse covariance matrix $\mtOmg$.
Due in part to its role in undirected graphical model selection, 
the problem of inverse covariance matrix estimation has received significant attention over the years. A popular approach for inverse covariance estimation, under high-dimensional settings, is the $\ell_1$-penalized Gaussian maximum likelihood estimate (MLE) studied by 
\cite{yuan2007model}, \cite{banerjee2008model}, and \cite{friedman_sparse_2008}, among others.
The $\ell_1$-penalized Gaussian MLE estimate of the inverse covariance matrix has attractive theoretical guarantees as shown by \cite{ravikumar_high-dimensional_2011}. 
However, the elementwise $\ell_{\infty}$ guarantees for the inverse covariance estimate obtained by \cite{ravikumar_high-dimensional_2011}
require an edge-based mutual incoherence condition that is quite restrictive. Many algorithms have been developed in the recent past for solving the $\ell_1$-penalized Gaussian MLE problem \cite{hsieh_big_2013, hsieh2012divide, rolfs2012iterative, johnson2012high}. While, technically, these algorithms can be  used in the first phase of our algorithm to estimate the inverse covariance matrix, in this paper we use the method called CLIME, developed by \cite{cai_constrained_2011}. The primary motivation behind using CLIME is that the theoretical guarantees obtained by Cai et al. \cite{cai_constrained_2011} does not require the edge-based mutual incoherence condition. Further, CLIME is computationally attractive because it computes $\mhOmg$ columnwise by solving $p$ independent linear programs. Even though the CLIME estimator $\mhOmg$ is not guaranteed to be positive-definite (it is positive-definite with high probability) it is suitable for our purpose since we use $\mhOmg$ only for identifying terminal vertices. Next, we briefly describe the CLIME method for inverse covariance estimation and instantiate the theoretical results of  \cite{cai_constrained_2011} for our purpose.

The CLIME estimator $\mhOmg$ is obtained as follows. First, we compute 
a potentially non-symmetric estimate $\bar{\mOmg} = (\bomg_{i,j})$ by solving the following:
\begin{align}
\bar{\mOmg} = \argmin_{\mOmg \in \R^{p \times p}} \Abs{\mOmg}_1 
	\text{ s.t. } \Abs{\mSig^n \mOmg - \mI}_{\infty} \leq \lambda_n, \label{eq:omegahat}
\end{align}
where $\lambda_n > 0$ is the regularization parameter, $\mSig^n \defeq (\nicefrac{1}{n}) \mX^T \mX$ is the empirical covariance matrix, and $\Abs{\cdot}_{1}$ (respectively $\Abs{\cdot}_{\infty}$) denotes elementwise $\ell_1$ (respectively $\ell_{\infty}$) norm. Finally, the symmetric estimator is obtained by selecting the smaller entry among $\bomg_{i,j}$ and $\bomg_{j,i}$, i.e.,
$\mhOmg = (\homg_{i,j})$, where 
$\homg_{i,j} = \bomg_{i,j} \Ind{\Abs{\bomg_{i,j}} < \Abs{\bomg_{j,i}}} + \bomg_{j,i} \Ind{\Abs{\bomg_{j,i}} \leq \Abs{\bomg_{i,j}}}$.
It is easy to see that \eqref{eq:omegahat} can be decomposed into $p$ linear programs as follows. 
Let $\bar{\mOmg} = (\vbomg_1, \ldots, \vbomg_p)$, then
\begin{align}
\vbomg_i = \argmin_{\vomg \in \R^p} \NormI{\vomg} \text{ s.t. } \Abs{\mSig^n \vomg - \ve_i}_{\infty} \leq \lambda_n, 
\label{eq:omegahat_colwise}
\end{align}
where $\ve_i = (e_{i,j})$ such that $e_{i,j} = 1$ for $j = i$ and $e_{i,j} = 0$ otherwise. The following lemma which follows from the results of \cite{cai_constrained_2011} and \cite{ravikumar_high-dimensional_2011}, bounds the maximum elementwise difference between $\mhOmg$ and the true precision matrix $\mtOmg$.
\begin{lemma}
\label{lemma:inv_cov}
Let $(\G^*, \Pf(\Ws^*, \sigma^2))$ be a GBN satisfying Assumption \ref{ass:causal_minimality},
with $\mSig^*$ and $\mtOmg$ being the ``true'' covariance and inverse covariance matrix over $X$, respectively.
Given a data matrix $\mX \in \R^{n \times p}$ of $n$ \iid samples drawn from $\Pf(\Ws^*, \sigma^2)$, compute $\mhOmg$ by solving \eqref{eq:omegahat}. Then, if the regularization parameter and number of samples satisfy:
\begin{gather*}
\lambda_n \geq \NormI{\mtOmg} \sqrt{(\nicefrac{C_1}{n}) \log (\nicefrac{4p^2}{\delta})}, \\
n \geq (\nicefrac{(16 \sigma^4 \NormI{\mtOmg}^4 C_1)}{\alpha^2}) \log (\nicefrac{(4p^2)}{\delta}),
\end{gather*}
with probability at least $1 - \delta$ we have that $|\mtOmg - \mhOmg|_{\infty} \leq \nicefrac{\alpha}{\sigma^2}$,
where $C_1 = 3200 \bigl(\max_{i} (\mSig^*_{i,i})^2\bigr)$ and $\delta \in (0,1)$. 
\end{lemma}

\begin{remark}
Note that in each column of the true precision matrix $\mtOmg$, at most $k$ entries are non-zero, where $k$ is the maximum Markov blanket size of a node in $\G$. Therefore, the $\ell_1$ induced (or operator) norm $\NormI{\mtOmg} = \BigO{k}$, and the sufficient number of samples required for the estimator $\mhOmg$ to be within $\alpha$ distance from $\mtOmg$, elementwise, with probability at least $1 - \delta$ is $\BigO{(\nicefrac{1}{\alpha^2}) k^4 \log (\nicefrac{p}{\delta})}$.
\end{remark}

\subsection{Estimating regression coefficients}
Given a GBN $(\G, \Pf(\Ws, \sigma^2))$ with the covariance and precision matrix over $X$
being $\mSig$ and $\mOmg$ respectively, 
the conditional distribution of $X_i$ given the variables in its Markov blanket is:
$X_i | (X_{\Si} = \vx) \sim \mathcal{N}((\vth_i)_{\Si}^T \vx,\, \nicefrac{1}{\Omega_{i,i}})$.
Let $\vth^i_{\Si} \defeq (\vth_i)_{\Si}$. This leads to the following generative model for $\mX_{*,i}$:
\begin{align}
\mX_{*,i} = (\mX_{*, \Si}) \vth^i_{\Si} + \veps'_i, \label{eq:gen_model_lsq}
\end{align}
where $\veps'_i \sim \mathcal{N}(0, \nicefrac{1}{\Omega_{i,i}})$ and 
$\mX_{l, \Si} \sim \mathcal{N}(\vect{0}, \mSig_{\Si, \Si})$ for all $l \in [n]$.
Therefore, for all $i \in [p]$, we obtain the estimator $\vhth^i_{\Si}$ of $\vth^i_{\Si}$ by solving the following
ordinary least squares (OLS) problem:
\begin{align}
\vhth^i_{\Si} &= \argmin_{\vbeta \in \R^{\Abs{\Si}}} \frac{1}{2n} \NormII{\mX_{*,i} - (\mX_{*,\Si}) \vbeta}^2 \notag\\
	&= (\mSig^n_{\Ss_i, \Ss_i})^{-1} \mSig^n_{\Ss_i, i} \label{eq:olsq}
\end{align} 
The following lemma bounds the approximation error between the true regression coefficients
and those obtained by solving the OLS problem.

\begin{lemma}
\label{lemma:reg_coeffs}
Let $(\G^*, \Pf(\Ws^*, \sigma^2))$ be a GBN with $\mSig^*$ and $\mtOmg$ being the true covariance and inverse covariance
matrix over $X$. Let $\mX \in \R^{n \times p}$ be the data matrix of $n$ \iid samples drawn from $\Pf(\Ws^*, \sigma^2)$. Let
$\Exp{}{X_i | (X_{\Si} = \vx)} = \vx^T \vth^i_{\Si}$, and let
$\vhth^i_{\Si}$ be the OLS solution obtained by solving \eqref{eq:olsq} for some $i \in [p]$.
Then, under Assumption \ref{ass:nonsingular}, and if the number of samples satisfy:
\begin{align*}
n \geq \frac{c \Abs{\Si}^{\nicefrac{3}{2}} (\NormInfty{\vth^i_{\Si}} + \nicefrac{1}{\Abs{\Si}})}{\eigmin(\mSig^*_{\Si, \Si}) \alpha}
	\log \left(\frac{4 \Abs{\Si}^2}{\delta} \right),
\end{align*}
we have that,
$\NormInfty{\vth^i_{\Si} - \vhth^i_{\Si}} \leq \alpha$ with probability at least $1 - \delta$, for some 
$\delta \in (0, 1)$, with $c$ being an absolute constant.
\end{lemma}
\subsection{Our algorithm}
Algorithm \ref{alg:main} presents our algorithm for learning GBNs. Throughout the algorithm we use as indices the true label of a node.
We first estimate the inverse covariance matrix, $\mhOmg$, in line \ref{line:omega_hat}. In line \ref{line:markov_blanket}
we estimate the Markov blanket of each node.
Then, we estimate $\widehat{\theta}_{i,j}$ for all $i$ and ${j \in \Sh_i}$,
and compute the maximum per-node ratios 
$r_i = \Abs{-\nicefrac{\widehat{\Omega}_{i,j}}{\widehat{\theta}_{i,j}}}$ in lines \ref{line:ri_start} -- \ref{line:ri_end}. We then identify as terminal vertex the node for which $r_i$ is minimum and remove it from the collection of variables
(lines \ref{line:term_vertex_1} and \ref{line:term_vertex_2}).
Each time a variable is removed, we perform a rank-1 update of the precision matrix 
(line \ref{line:precision_update}) and also update the regression coefficients of the nodes in its Markov blanket (lines \ref{line:update_start} -- \ref{line:update_end}). We repeat this process of identifying and removing  terminal vertices until the causal order has been completely determined.
Finally, we compute the DAG structure and parameters
by regressing each variable against variables that are in its Markov blanket which also precede it in the causal order (lines \ref{line:learn_start} -- \ref{line:learn_end}).
\begin{algorithm}[!ht]
\caption{Gaussian Bayesian network structure learning algorithm.}
\label{alg:main}
\begin{algorithmic}[1]
\Require Data matrix $\mX \in \R^{n \times p}$.
\Ensure $(\Gh, \Wsh)$.
\State $\mhB \gets \vect{0} \in \R^{p \times p}$.
\State $\vz \gets \varnothing$. \Comment{$\vz$ stores the causal order.}
\State $\vr \gets \varnothing$. 
\State $\Vs \gets [p]$. \Comment{Remaining vertices.}
\State $\mSig^n \gets (\nicefrac{1}{n}) \mX^T \mX$.
\State Compute $\mhOmg$ using the CLIME estimator. \label{line:omega_hat}
\State $\forall i \in [p]$, compute $\Sh_i = 
	\Set{j \in \mi \:|\: \Abs{\widehat{\Omega}_{i,j}} > 0}$. \label{line:markov_blanket}
\For{$i \in 1, \ldots, p$} \label{line:ri_start}
	\State Compute $\vhth^i_{\Sh_i} \defeq (\vhth_i)_{\Sh_i} = 
		(\mSig^n_{\Sh_i, \Sh_i})^{-1} \mSig^n_{\Sh_i, i} $.
	\State $r_i \gets \max\Set{\Abs{-\nicefrac{\widehat{\Omega}_{i,j}}{\widehat{\theta}_{i,j}}} \: | \: j \in \Sh_i} $.
\EndFor  \label{line:ri_end}
\For{$t \in 1 \ldots p - 1$}
	\State $i \gets \argmin(\vr)$. \Comment{$i$ is a terminal vertex}. \label{line:term_vertex_1}
	\State Append $i$ to $\vz$; $\Vs \gets \Vs \setminus \Set{i}$; $r_i \gets +\infty$. \label{line:term_vertex_2}
	\State $\mhOmg \gets \mhOmg_{\mi, \mi} - 
		(\nicefrac{1}{\widehat{\Omega}_{i,i}}) (\mhOmg_{\mi,i})(\mhOmg_{i, \mi})$
		\label{line:precision_update}.
	\For{$j \in \Sh_i$} \label{line:update_start}
		\State $\Sh_j \gets \Set{l \neq j \:|\: \Abs{\widehat{\Omega}_{j,l}} > 0}$.
		\State Compute $\vhth^j_{\Sh_j} \defeq (\vhth_j)_{\Sh_j}
			= (\mSig^n_{\Sh_j, \Sh_j})^{-1} \mSig^n_{\Sh_j, j}$.
		\State $r_j \gets \max\Set{\Abs{-\nicefrac{\widehat{\Omega}_{j,l}}{\widehat{\theta}_{j,l}}} \: | \: l \in \Sh_j}$.
	\EndFor \label{line:update_end}
\EndFor
\State Append the remaining vertex in $\Vs$ to $\vz$.
\For{$i \in 2, \ldots, p$} \label{line:learn_start}
	\State $\Sh_{z_i} \gets \Set{z_j | j \in [i - 1]} \intersection 
		\Set{j \in [p] \: |\:  j \neq z_i \Land \Abs{\widehat{\Omega}_{z_i,j}} > 0}$.
	\State Compute $\vhth = (\mSig^n_{\Sh_{z_i}, \Sh_{z_i}})^{-1} \mSig^n_{\Sh_{z_i}, z_i}$ \label{line:theta}.
	\State $\hPar{z_i}{} \gets \Sp(\vhth)$. 	\label{line:parent_set}
	\State $\mhB_{z_i, \hPar{z_i}{}} \gets \vhth_{\hPar{z_i}{}}$.
\EndFor
\State $\widehat{\Es} \gets \Set{(i, j) \:|\: \widehat{B}_{i,j} \neq 0}$, 
	$\Wsh \gets \Set{\widehat{B}_{i,j} | (i,j) \in \widehat{\Es}}$, and 
	$\Gh \gets ([p], \widehat{\Es})$.  \label{line:learn_end}
\end{algorithmic}
\end{algorithm}

In order to obtain our main result for learning GBNs we first derive the following technical lemma which states that if
the data comes from a GBN that satisfies Assumptions \ref{ass:causal_minimality} -- \ref{ass:nonsingular},
then removing a terminal vertex results in a GBN that still satisfies Assumptions \ref{ass:causal_minimality} -- \ref{ass:nonsingular}.
\begin{lemma}
\label{lemma:assumptions}
Let $(\G, \Pf(\Ws, \sigma^2))$ be a GBN satisfying Assumptions \ref{ass:causal_minimality} -- \ref{ass:nonsingular} and let $\mOmg$ be the precision matrix. Let $\mX \in \R^{n \times p}$ be a data matrix of $n$ \iid samples drawn from $\Pf(\Ws, \sigma^2)$,
and let $i$ be a terminal vertex in $\G$. 
Denote by $\G'=(\Vs', \Es')$ and $\Ws' = \Set{w_{i,j} \in \Ws \:|\: (i,j) \in \Es'}$ the graph and set of edge weights, 
respectively, obtained by removing the node $i$ from $\G$. Then,
$\mX_{j,\mi} \sim \Pf(\Ws', \sigma^2) \; \forall j \in [n]$, and the GBN $(\G', \Pf(\Ws', \sigma^2))$ satisfies
Assumptions \ref{ass:causal_minimality} -- \ref{ass:nonsingular}. Further, the inverse covariance matrix $\mOmg'$ and the covariance matrix $\mSig'$ for the GBN $(\G', \Pf(\Ws', \sigma^2))$ satisfy (respectively):
$\mOmg' = \mOmg - (\nicefrac{1}{\Omega_{i,i}}) \mOmg_{*,i} \mOmg_{i,*}$
and $\mSig' = \mSig_{\mi, \mi}$.
\end{lemma}

\begin{theorem}
\label{thm:main}
Let $\widehat{\G} = ([p], \widehat{\Es})$ and $\widehat{\Ws}$ be the DAG and edge weights, respectively, returned by Algorithm \ref{alg:main}.
Under the assumption that the data matrix $\mX$ was drawn from a GBN $(\G^*, \Pf(\Ws^*, \sigma^2))$ with $\G^* = ([p], \Es^*)$,
$\mSig^*$ and $\mtOmg$ being the ``true'' covariance and inverse
covariance matrix respectively, and
satisfying Assumptions \ref{ass:causal_minimality} -- \ref{ass:nonsingular}; if the regularization
parameter is set according to Lemma \ref{lemma:inv_cov}, and if the number of samples satisfies the condition:
\begin{align*}
&n \geq c \left( \frac{\sigma^4 \NormI{\mtOmg}^4 \cmax}{\alpha^2} + \frac{k^{(\nicefrac{3}{2})}(\wtlmax + \nicefrac{1}{k})}{\cmin \alpha} \right) 
	 \log \left(\frac{24 p^2 (p - 1)}{\delta} \right),
\end{align*} 
where $c$ is an absolute constant, $\wtlmax \defeq \max \Set{\Abs{\wtl_{i,j}} \:|\: i \in \Vs[m, \tau] \Land 
		j \in \Ss_i[m, \tau] \Land m \in [p] \Land \tau \in \Ts_{\G} }$
with $\wtl_{i,j}$ being the effective influence between $i$ and $j$ \eqref{eq:effective_influence}, $\cmax = \max_{i \in p} (\mSig^*_{i,i})^2$,
and $\cmin = \min_{i \in [p]} \eigmin(\mSig^*_{\Ss_i, \Ss_i})$, then, $\widehat{\Es} = \Es^*$ and $\forall (i,j) \in \widehat{\Es},\: \Abs{\widehat{w}_{i,j} - w^*_{i,j}} \leq \alpha$ with probability at least $1 - \delta$ for some $\delta \in (0, 1)$ and $\alpha > 0$.
\end{theorem}
The CLIME estimator of the precision matrix can be computed in polynomial time and the OLS steps take $\BigO{p k^3}$ time. 
Therefore our algorithm is polynomial time. For more details see Appendix \ref{app:appendix_discussion}.

\section{Experiments}
In this section we study the empirical performance of our method on synthetic and real-world data.
In the first set of experiments we seek to empirically characterize the number of samples needed by our method for learning the DAG structure of a GBN exactly. We sample a random DAG structure $\G^*$ over $p$ nodes by first generating an Erd\H{o}s-R\'{e}nyi undirected graph where each edge is sampled independently with probability $q$. Then, we randomly select a permutation of the vertex set $[p]$ and direct the edges as $i \rightarrow j$ if the node $i$ appears before node $j$ in the permutation. We then generate a GBN $(\G^*, \Pf(\Ws^*, \sigma^2))$ by setting the noise variance $\sigma^2 = 0.8$ for all nodes and randomly setting the edge weights to $w^*_{i,j} = \pm \nicefrac{1}{2}$ with probability $\nicefrac{1}{2}$. 
To avoid numerical issues we discarded GBNs where the minimum eigenvalue of 
the inverse covariance matrix was less than $0.05$. Further, we verified that across thousands of randomly sampled GBNs, RSAF was satisfied with $\alpha$ varying between between $0.25$ to $0.5$. 
After sampling a GBN, we sample a data set of $n$ samples and learn a GBN $(\Gh, \Pf(\Wsh, \widehat{\sigma}^2))$. Finally, we estimate the probability $\Prob{\G^* = \Gh}$ by computing the fraction of times the learned DAG structure $\Gh$ matched the true DAG structure $\G^*$ exactly, across 30 randomly sampled GBNs. We repeated the experiment for $p \in \Set{50, 100, 150, 200}$ with $q \in \Set{0.01, 0.005, 0.0033, 0.0025}$ (correspondingly). The number of samples was set to $C k^2 \log p$, where $C$ was the control parameter and was chosen to be in $\Set{1, 20, 40, 80, 100, 120}$, and $k$ was the maximum size of the Markov blanket across all nodes in the sampled DAG $\G^*$. The mean and maximum value of $k$ (across 30 runs) for the different choices of $p$ was $\Set{3.2, 3.68, 4.12, 4.39}$ and $\Set{7, 10, 7, 9}$ respectively. The regularization parameter was set to $\lambda_n = 0.5 k \sqrt{\nicefrac{(\log p)}{n}}$, as prescribed by Lemma $\ref{lemma:inv_cov}$. Figure \ref{fig:struct_recovery} shows the results of the structure and parameter recovery experiments. We can see that the $\log p$ scaling as prescribed by Theorem \ref{thm:main} holds in practice. 
\begin{figure}[htbp]
\begin{center}
\includegraphics[width=0.5\linewidth]{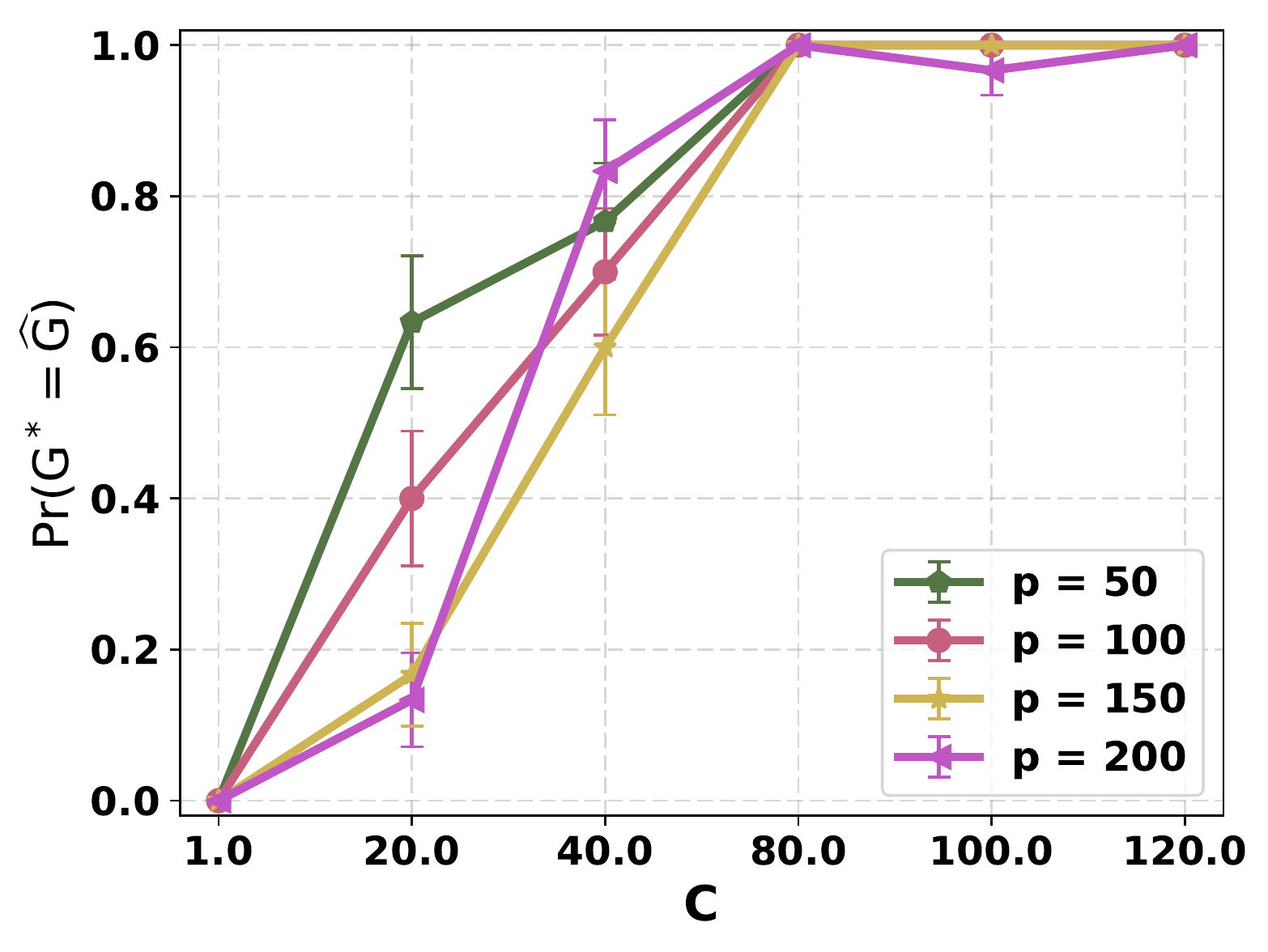}%
\includegraphics[width=0.5\linewidth]{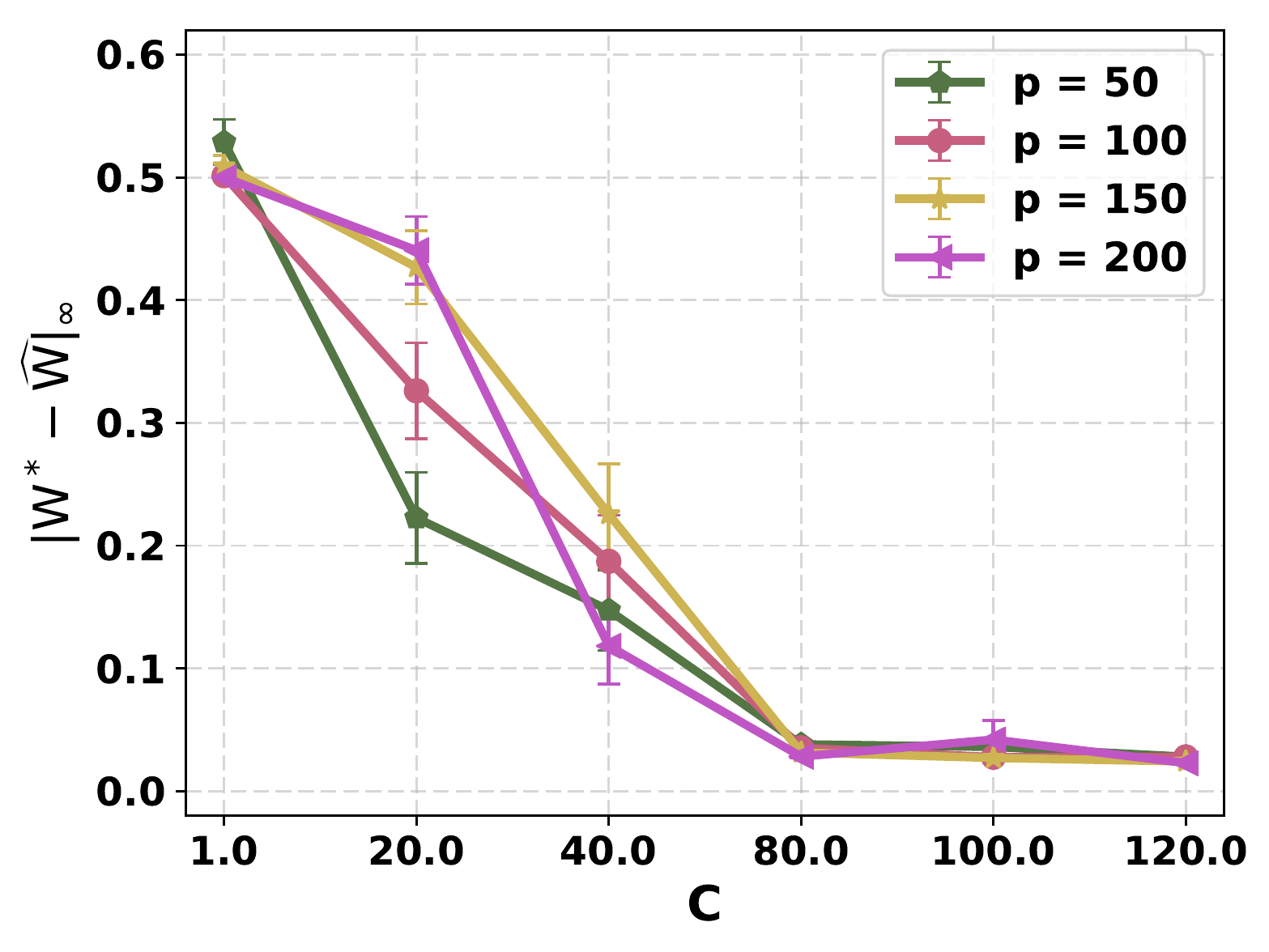}
\end{center}
\vspace*{-0.25in}
\caption{{\small (Left) Probability of correct structure recovery vs. number of samples, where the latter is set to $C k^2 \log p$ with $C$ being the control parameter and $k$ being the maximum Markov blanket size. (Right) The maximum absolute difference
between the true parameters and the learned parameters vs. number of samples. \label{fig:struct_recovery}}}
\end{figure}

In the final set of simulation experiments, we compare the
performance of our algorithm against state-of-the-art methods
for learning GBNs. Once again, we sampled DAGs according to the procedure described in the previous paragraphs. We considered three methods for comparison: 
PC algorithm for learning GBNs by \cite{kalisch_estimating_2007},
the greedy equivalence search (GES) algorithm by \cite{chickering_optimal_2003}, and the max-min hill climbing (MMHC) algorithm by \cite{tsamardinos2006max}. The first two of the three algorithms estimate the Markov equivalence class and therefore return
a completed partially directed acyclic graph (CPDAG). However, in our experiments, the sampled DAGs belong to Markov equivalence classes of size 1.
Therefore, the CPDAGs should ideally have no undirected edges. 
We do not compare against the $\ell_0$ penalized MLE algorithm by \cite{Peters2014} for the equal variance case, which is an exact algorithm, since it searches through the \emph{super-exponential} space of all DAGs and therefore does not scale beyond 20 nodes.
The GES algorithm uses the $\ell_0$-penalized Gaussian MLE score proposed by \cite{Peters2014} to greedily search for the best structure. 
We used the R package \emph{pcalg} for the implementation of the PC and GES algorithms, and the \emph{bnlearn} package for the implementation of the MMHC algorithm. MMHC and PC take an additional tuning parameter $\alpha$ which is the desired significance level for the individual conditional independence tests. We tested values of $\alpha \in \Set{0.01, 0.001, 0.0001}$ and found that $\alpha = 0.0001$ gave the best results on an average. The number of samples was set to $120 k^2 \log p$ and the regularization parameter for our method was set to $2 \sqrt{\nicefrac{(\log p)}{n}}$. We also used both the BIC score and the Bayesian Gaussian equivalent (BGe) score for the MMHC algorithm and found that BGe produced better results on an average. We computed the mean precision, recall, and running time in seconds, for each method, across 30 randomly sampled GBNs. Precision is defined as the fraction of all predicted (directed) edges that are actually present in the true DAG, while recall is defined as the fraction of directed edges in the true DAG that the method was able to recover. All methods were run on a single core of Intel$^{\tiny{\textregistered}}$ Xeon$^{\tiny{\textregistered}}$ running at 3.00 Ghz. The results are shown in Table \ref{tab:comparison}. We can see that our method outperforms existing methods in terms of precision and recall. Moreover, our method, which we implemented in Python, is the fastest among all methods for $p \leq 100$, and is always faster than MMHC. Among, MMHC, GES and PC, the PC algorithm performed the best since it is an exact algorithm. However, the PC algorithm failed to direct many edges as is evident from its low precision score. Note that we used the function \emph{udag2pdag} in the R package \emph{pcalg}, to convert the undirected skeleton returned by the PC algorithm to a CPDAG.
Please see Appendix \ref{app:experiments} for experiments on real-world data,
non-equal variances, and a comparison of our algorithm with the PC algorithm on
a non-faithful DAG.
\begin{table}[htbp]
\centering
\begin{small}
\begin{tabular}{cccc}
\hline
Method & Precision & Recall & Seconds \\

\hline
\multicolumn{4}{c}{p = 50} \\
\hline
PC & 0.587 $\pm$ 0.015 & 0.996 $\pm$ 0.004 & 0.177 $\pm$ 0.013  \\
GES & 0.206 $\pm$ 0.014 & 0.396 $\pm$ 0.031 & 0.206 $\pm$ 0.025 \\
MMHC & 0.581$\pm$ 0.038 & 0.583$\pm$ 0.038 & 0.460$\pm$ 0.049 \\
Ours & \textbf{ 1.000 $\pm$ 0.000} & \textbf{1.000 $\pm$ 0.000} & \textbf{0.089 $\pm$ 0.005} \\

\hline
\multicolumn{4}{c}{p = 100} \\
\hline	
PC	& 0.587 $\pm$ 0.008 & 0.999 $\pm$ 0.001 & 0.570 $\pm$ 0.044 \\
GES & 0.204 $\pm$ 0.013 & 0.372 $\pm$ 0.020 & 0.557 $\pm$ 0.045 \\
MMHC & 0.529$\pm$ 0.019 & 0.533$\pm$ 0.019 & 1.417$\pm$ 0.141 \\
Ours & \textbf{ 1.000 $\pm$ 0.000} & \textbf{1.000 $\pm$ 0.000} & \textbf{0.534 $\pm$ 0.004} \\

\hline
\multicolumn{4}{c}{p = 150} \\
\hline
PC & 0.572$\pm$ 0.006 & 0.996$\pm$ 0.002 & 1.392$\pm$ 0.043 \\
GES & 0.162$\pm$ 0.009 & 0.333$\pm$ 0.017 & \textbf{1.031$\pm$ 0.036	} \\	
MMHC & 0.566$\pm$ 0.014 & 0.577$\pm$ 0.015 & 2.934$\pm$ 0.241 \\
Ours & \textbf{1.000$\pm$ 0.000} & \textbf{1.000$\pm$ 0.000} & 1.988$\pm$ 0.010 \\

\hline
\multicolumn{4}{c}{p = 200} \\
\hline
PC  & 0.573$\pm$ 0.005 & 0.997$\pm$ 0.001 & 1.876$\pm$ 0.080 \\
GES & 0.143$\pm$ 0.005 & 0.310$\pm$ 0.011 & \textbf{1.610$\pm$ 0.077} \\	
MMHC & 0.582$\pm$ 0.012 & 0.593$\pm$ 0.012 & 5.511$\pm$ 0.355 \\
Ours & \textbf{1.000$\pm$ 0.000} & \textbf{1.000$\pm$ 0.000} & 5.130$\pm$ 0.030 \\
\hline		
\end{tabular}
\end{small}
\vspace*{-0.05in}
\caption{\small Performance of different algorithms across 30 randomly sampled GBNs for each value of $p \in \Set{50, 100, 150, 200}$. Numbers in bold are the best for each metric across different algorithms. Our method always recovers the true DAG structure exactly. Each sampled GBN belonged to a Markov equivalence class of size 1. \label{tab:comparison}}	
\end{table}

\vspace*{-0.2in}
\paragraph{Concluding remarks.}
There are several ways of extending our current work. While the algorithm developed in the paper is specific to equal noise-variance case, we believe our theoretical analysis can be extended to the non-identifiable case to show that our algorithm, under some suitable conditions, can recover one of the Markov-equivalent DAGs. It would be also interesting to explore if some of the ideas developed herein can be extended to binary or discrete Bayesian networks.

\begin{small}
\bibliographystyle{natbib}
\bibliography{paper}
\end{small}
\clearpage

\begin{appendices}
\section{Detailed Proofs}
\label{app:detailed_proofs}
\begin{proof}[Proof of Lemma \ref{lemma:theta}]
Consider the conditional distribution of $X_i | (X_{\mi} = \vx_{\mi})$. From standard results for Gaussians (see e.g., Chapter 2 of \cite{Bishop2006}), we have that:
\begin{gather}
X_i | (X_{\mi} = \vx_{\mi}) = \vth_i^T \vx_{\mi} + \varepsilon'_i, \text{ where }\\
\vth_i = \mSig_{i,\mi}(\mSig_{\mi,\mi})^{-1}  = -\frac{\mOmg_{i,\mi}}{\Omega_{i,i}} 
\text{ and } \varepsilon'_i \sim \mathcal{N}(0, \Omega_{i,i}^{-1}).
\end{gather}
From \eqref{eq:precision_mat} we have that:
\begin{align}
 \Omega_{i,j} &= \frac{1}{\sigma^2}(\mI_{*i} - \mB_{*i})^T(\mI_{*j} - \mB_{*j}) \notag \\
	&= \frac{1}{\sigma^2}(\mB_{*i}^T\mB_{*j} - B_{i,j} - B_{j,i})  \quad (\forall j \in \mi), \label{eq:omega_ij} \\ 
\Omega_{i,i} &= \frac{1}{\sigma^2}(\mI_{*i} - \mB_{*i})^T(\mI_{*i} - \mB_{*i}) =  \frac{1}{\sigma^2} (1 + \mB_{*i}^T\mB_{*i}), \label{eq:omega_ii}
\end{align}
where in \eqref{eq:omega_ij} we used the fact that $\mI_{*j}$ is a vector of all zeros except for a one at the $j$-th index
and in \eqref{eq:omega_ii} we used the fact that $\mB_{*i}^T \mI_{*i} = B_{i,i} = 0$. Combining \eqref{eq:omega_ij} and \eqref{eq:omega_ii} we prove our claim.
\end{proof}

\begin{proof}[Proof of Lemma \ref{lemma:terminal_vertex}]
The forward direction ($\Rightarrow$) follows directly from \eqref{lemma:theta} and the fact that for a terminal vertex $i$, $\mB_{*,i} = \vect{0}$. 

Now consider the reverse direction ($\Leftarrow$). In the first case, we have $\vth_i = - \sigma^2 \mOmg_{i,*} \neq \vect{0}$. Then, there exists a $j \in \mi$ such that $\theta_{ij} = - \sigma^2 \Omega_{i,j} \neq 0$, which implies, from Lemma \ref{lemma:theta}, that $\mB_{*,i} = \vect{0}$ and therefore $i$ is a terminal vertex. 

In the second case, we have $\vth_i = - \sigma^2 \mOmg_{i,*} = \vect{0}$.
We will proceed with a proof by contradiction. Assume that $i$ is not a terminal vertex. Then, there exists an edge $(j, i) \in \Es$. Further, since $\vth_i = \vect{0}$, we must have, from Lemma \ref{lemma:theta}, that $B_{i,j} + B_{j,i} = \mB^T_{*,i} \mB_{*,j} \neq 0$.
Therefore, nodes $i$ and $j$ must have common children. Denote the set of common children of $i$ and $j$ by $\mathsf{C} \defeq \Chi{i}{} \intersection \Chi{j}{}$.
There must be a node $k \in \mathsf{C}$ such that nodes $i$ and $k$ in turn do not have any common children, otherwise
the DAG $\G$ must have a cycle. Now if $i$ and $k$ do not have any common children, then $\theta_{ik} = -\sigma^2 \Omega_{i,k} \neq 0$,
which is a contradiction. Therefore, $i$ must be a terminal vertex.
\end{proof}

\begin{proof}[Proof of Lemma \ref{lemma:inv_cov}]
From Theorem 6 of \cite{cai_constrained_2011} we get that 
$|\mtOmg - \mhOmg|_{\infty} \leq 4 \NormI{\mtOmg} \lambda_n \leq  \nicefrac{\alpha}{\sigma^2}$,
if $\lambda_n \leq \nicefrac{\alpha}{4 \sigma^2 \NormI{\mtOmg}}$. The lower bound requirement on $\lambda_n$ comes from Theorem 6 of \cite{cai_constrained_2011}: $\lambda_n \geq \NormI{\mtOmg} |\mSig^* - \mSig^n|_{\infty}$.

Next, we show that the empirical covariance matrix $\mSig^n$
is concentrated around the true covariance matrix $\mSig^*$, elementwise, by using the results of \cite{ravikumar_high-dimensional_2011}. Note that $\nicefrac{X_i}{\sqrt{\mSig^*_{i,i}}} \sim \mathcal{N}(0, 1)$. Therefore, from Lemma 1 of \cite{ravikumar_high-dimensional_2011}, we have for a fixed $i$ and $j$:
\begin{align*}
\Prob{|\mSig^*_{i,j} - \mSig^n_{i,j}| \geq \varepsilon'} &\leq 4 \exp \left\{\frac{-n \varepsilon'^2}{C_1}\right\}.
\end{align*}
Therefore, by a union bound over all entries of $\mSig^n$, we have:
\begin{align*}
\implies \Prob{|\mSig^* - \mSig^n|_{\infty} \leq \varepsilon'} &\geq 1 - 4 p^2 \exp \left\{\frac{-n \varepsilon'^2}{C_1}\right\}.
\end{align*}
By setting $4 p^2 \exp(\nicefrac{-n \varepsilon'^2}{C_1}) = \delta$ and solving for $\varepsilon'$ we get that
the following holds with probability at least $1 - \delta$:
\begin{align*}
|\mSig^* - \mSig^n|_{\infty} &\leq \sqrt{(\nicefrac{C_1}{n}) \log \Bigl(\frac{4p^2}{\delta}\Bigr)}
\end{align*}
The lower bound on the number of samples comes from ensuring that lower bound on $\lambda_n$ is less than the upper bound $\nicefrac{\alpha}{4 \sigma^2 \NormI{\mtOmg}}$, i.e., $\NormI{\mtOmg} \sqrt{(\nicefrac{C_1}{n}) \log (\nicefrac{4p^2}{\delta})} \leq \nicefrac{\alpha}{4 \sigma^2 \NormI{\mtOmg}}$.
\end{proof}

\begin{proof}[Proof of Lemma \ref{lemma:reg_coeffs}]
Let $\mSig^n \defeq (\nicefrac{1}{n}) \mX^T \mX$, be the sample covariance matrix. We first
lower bound the minimum eigenvalue of the sample covariance matrix, $\eigmin(\mSig^n_{\Si, \Si})$, which will
be used later on in the proof. For the purpose of this proof, we will simply write $\Ss$ instead of $\Si$,
since we will derive our results for the $i$-th node for any $i \in [p]$.
\begin{align}
\eigmin(\mSig^n_{\Ss, \Ss}) = \min_{\NormII{\vy} = 1} \frac{1}{n} \NormII{(\mX_{*,\Ss}) \vy}^2 \geq \frac{\smin^2(\mX_{*,\Ss})}{n}, %
\label{eq:min_eig1}
\end{align}
where $\smin(.)$ (respectively $\smax(.)$) denotes the minimum (respectively maximum) singular value. 
Now note that for any $l \in [n]$, the $\Abs{\Ss}$-dimensional
vector $\mX_{l,\Ss} (\mSig^*_{\Ss,\Ss})^{-\nicefrac{1}{2}}$ is drawn from an isotropic Gaussian distribution.
Therefore, from Theorem 5.39 of \cite{vershynin_introduction_2010} we have:
\begin{align*}
\smin(\mX_{*,\Ss} (\mSig^*_{\Ss,\Ss})^{-\nicefrac{1}{2}}) \geq \sqrt{n} - C \sqrt{\Abs{\Ss}} - t,
\end{align*}
with probability at least $1 - 2 \exp (-ct^2)$, where $C$ and $c$ are absolute constants that depend
only on the sub-Gaussian norm $\Norm{X_{\Ss} (\mSig^*_{\Ss,\Ss})^{-\nicefrac{1}{2}})}_{\psi_2}$. 
Next, using the fact that 
$\smin(\mX_{*,\Ss} (\mSig^*_{\Ss,\Ss})^{-\nicefrac{1}{2}}) \leq \smin(\mX_{*,\Ss}) \smax((\mSig^*_{\Ss,\Ss})^{-\nicefrac{1}{2}})$,
we get:
\begin{align}
\smin(\mX_{*, \Ss}) &\geq \frac{\sqrt{n} - C \sqrt{\Abs{\Ss}} - t}{\smax((\mSig^*_{\Ss,\Ss})^{-\nicefrac{1}{2}}))} \notag \\
&= \smin((\mSig^*_{\Ss,\Ss})^{-\nicefrac{1}{2}}))(\sqrt{n} - C \sqrt{\Abs{\Ss}} - t). \label{eq:min_eig2}
\end{align}
Finally, from \eqref{eq:min_eig1} and \eqref{eq:min_eig2}, we have that:
\begin{align}
\eigmin(\mSig^n_{\Ss,\Ss}) &\geq \eigmin(\mSig^*_{\Ss,\Ss}) \left(1 - C\sqrt{\frac{\Abs{\Ss}}{n}} - \frac{t}{\sqrt{n}} \right)^2 \notag \\
  & \geq \frac{\eigmin(\mSig^*_{\Ss,\Ss})}{4} \label{eq:min_eig3}
\end{align}
with probability at least $1 - 2 \exp(-cn)$, where $c$ is an absolute constant, and the second line follows from controlling the second term in side the parenthesis to be at most $\nicefrac{1}{2}$.

Next, from the normal equations of least squares, we have that $\vhth^i_{\Ss} = (\mSig^n_{\Ss, \Ss})^{-1} \mSig^n_{\Ss, i}$,
while the true coefficient vector satisfies: $\vth^i_{\Ss} = (\mSig^*_{\Ss, \Ss})^{-1} \mSig^*_{\Ss, i}$.
For notational simplicity, let us write $\vth_{\Ss}$ and $\vhth_{\Ss}$, respectively, instead of $\vth^i_{\Ss}$ and $\vhth^i_{\Ss}$.
From the entry-wise tail bounds for the sample covariance matrix derived by \cite{ravikumar_high-dimensional_2011},
we have that:
\begin{align}
\NormInfty{\mSig^*_{\Ss, \Ss} \vth_{\Ss} - \mSig^n_{\Ss, \Ss} \vhth_{\Ss}} = \NormInfty{\mSig^*_{\Ss, i} - \mSig^n_{\Ss, i}} \leq \varepsilon', %
\label{eq:dels_1}
\end{align}
with probability at least $1 - 4 \Abs{\Ss}^2 \exp(\nicefrac{(-n\varepsilon'^2)}{C_1})$.
Let $\dels \defeq \vhth_{\Ss} - \vth_{\Ss}$. Then, using the reverse triangle inequality we get:
\begin{align}
&\NormInfty{\mSig^*_{\Ss, \Ss} \vth_{\Ss} - \mSig^n_{\Ss, \Ss} \vhth_{\Ss}}
\notag \\
 &\qquad = \NormInfty{(\mSig^*_{\Ss, \Ss} - \mSig^n_{\Ss, \Ss}) \vth_{\Ss} - \mSig^n_{\Ss, \Ss} \dels} \notag \\
&\qquad \geq \NormInfty{\mSig^n_{\Ss, \Ss} \dels} - \Abs{\Ss} \varepsilon' \NormInfty{\vth_{\Ss}}. \label{eq:dels_2}	
\end{align}
Next, from \eqref{eq:dels_1} and \eqref{eq:dels_2} we get:
\begin{gather*}
\NormII{\mSig^n_{\Ss, \Ss} \dels} \leq \Abs{\Ss}^{\nicefrac{3}{2}} \varepsilon'(\NormInfty{\vth_{\Ss}} + \nicefrac{1}{\Abs{\Ss}}) \\
\implies \NormII{\dels} \leq 
	\frac{\Abs{\Ss}^{\nicefrac{3}{2}} \varepsilon'(\NormInfty{\vth_{\Ss}} + \nicefrac{1}{\Abs{\Ss}})}{\eigmin(\mSig^n_{\Ss, \Ss})} \\
\implies \NormInfty{\dels} \leq 
	\frac{4 \Abs{\Ss}^{\nicefrac{3}{2}} \varepsilon'(\NormInfty{\vth_{\Ss}} + \nicefrac{1}{\Abs{\Ss}})}{\eigmin(\mSig^*_{\Ss, \Ss})} 
	\leq \alpha,
\end{gather*}
with probability at least $1 - 4\Abs{\Ss}^2 \exp\bigl(- \frac{n\, c \,\alpha \eigmin(\mSig^*_{\Ss,\Ss})}
		{\Abs{\Ss}^{\nicefrac{3}{2}} (\NormInfty{\vth_{\Ss}} + \nicefrac{1}{\Abs{\Ss}})} \bigr)$,
where the second line follows from the fact that $\mSig^n_{\Ss, \Ss}$ is full rank (with high probability),
and the last line follows from \eqref{eq:min_eig3} and the fact that $\NormInfty{.} \leq \NormII{.}$.
Finally, by controlling the probability of error to be at most $\delta$, we derive the lower bound on the
number of samples.
\end{proof}

\begin{proof}[Proof of Lemma \ref{lemma:assumptions}]
Let $\mB$ be the weight matrix corresponding to the edge weights $\Ws$, and let $\mB' = \mB_{\mi, \mi}$ 
denote the weight matrix corresponding to the edge weights $\Ws'$. Consider any topological order $\tau \in \Ts_{\G}$.
We will denote by $(i)_{\tau}$ the $i$-th node in the toplogical order $\tau \in \Ts_{\G}$.
The joint distribution over
$(\mX_{*,(1)_{\tau}}, \ldots, \mX_{*,(p)_{\tau}})$ is given by a linear SEM where $\mX_{*,(i)_{\tau}}$ depends only on the 
variables occurring before the variable $(i)_{\tau}$ in the topological order $\tau$:
\begin{align*}
\mX_{*,(i)_{\tau}} = \sum_{j=1}^{i-1} B_{(i)_{\tau}, (j)_{\tau}} \mX_{*,(j)_{\tau}} + \varepsilon,
\end{align*}
with $\varepsilon \sim \mathcal{N}(0, \sigma^2)$. Therefore, if we remove a terminal vertex, then the linear equations that describe the remaining variables do not change. Thus, if 
$\mOmg'$ and $\mSig'$ denote the precision and covariance matrix after removing node $i$, which is
a terminal node, then:
\begin{align*}
\mOmg' &= \frac{1}{\sigma^2} (\mI - \mB')^T(\mI - \mB') \\
&= \frac{1}{\sigma^2} (\mI - \mB_{\mi, \mi})^T(\mI - \mB_{\mi, \mi}) \\
\mSig' &= \sigma^2 (\mI - \mB')^{-1}(\mI - \mB')^{-T} \\
&= \sigma^2 (\mI - \mB_{\mi, \mi})^{-1}(\mI - \mB_{\mi, \mi})^{-T}.
\end{align*}
The fact that $\Pf(\Ws', \sigma^2)$ is causal minimal (Assumption \ref{ass:causal_minimality}) and $\alpha$-RSAF 
(Assumption \ref{ass:faithfulness}) is self evident.
Next, using the fact that $\mSig' = \mSig_{\mi, \mi}$, we have:
\begin{align*}
0 < \eigmin(\mSig) &=
 \min_{\Set{\vy \in \R^p | \vy^T \vy = 1}} \vy^T \mSig \vy \\
 &\leq 
 \min_{\Set{\vy \in \R^p | \vy^T \vy = 1 \Land y_i = 0}} \vy^T \mSig \vy \\
&= \min_{\vy \in \R^{p-1}} \vy^T \mSig' \vy = \eigmin(\mSig').
\end{align*}
This proves that the distribution $\Pf(\Ws', \sigma^2)$ is non-singular (Assumption \ref{ass:nonsingular}). Finally, the precision matrix and the covariance matrix for $X_{\mi}$ is given by $\mOmg' = \mOmg - (\nicefrac{1}{\Omega_{i,i}}) \mOmg_{*,i} \mOmg_{i,*}$ and $\mSig' = \mSig_{\mi, \mi}$ respectively, which follows from standard results for marginalization of multivariate Gaussian distribution (see for instance Chapter 2 of \cite{Bishop2006}).
\end{proof}

\begin{proof}[Proof of Theorem \ref{thm:main}]
First note that the lower bound on the number of samples given by Theorem \ref{thm:main} subsumes the sample complexity requirement of inverse covariance estimation in Lemma \ref{lemma:inv_cov} 
ordinary least squares in Lemma \ref{lemma:reg_coeffs}. Next, by Assumption
\ref{ass:faithfulness}, we have that $\forall i \in [p],\: \Sh_i = \Ss_i$, with probability at least $1 - \delta$. Therefore, from Lemma \ref{lemma:reg_coeffs} $\NormInfty{\vth^i_{\Ss_i} - \vhth^i_{\Sh_i}} \leq \alpha$,
with probability at least $1 - 2 \delta$.

Next, from Lemmas \ref{lemma:inv_cov} and \ref{lemma:reg_coeffs}, and by our assumption that
$\Abs{\wtl_{i,j}} \geq 3 \alpha$, we have that 
for a terminal vertex $i$, the ratio $r_i$ is upper bounded as follows:
\begin{align*}
r_i &\leq \frac{\Abs{\wtl_{i,j}} + \alpha}{\sigma^2(\Abs{\wtl_{i,j}} - \alpha)} \\
	&\leq \frac{4 \alpha}{\sigma^2 (2 \alpha)} = \frac{2}{\sigma^2},
\end{align*}
where the second line follows from the fact that $\frac{\Abs{\wtl_{i,j}} + \alpha}{\sigma^2(\Abs{\wtl_{i,j}} - \alpha)}$
is a decreasing function of $\Abs{\wtl_{i,j}}$.
Similarly, if $i$ is a non-terminal vertex and has $c_i$ children, then the ratio
is lower bounded as follows:
\begin{align*}
r_i \geq \left( \frac{1}{\sigma^2} \right) 
	\frac{\Abs{\wtl_{i,j}} - \alpha}{\frac{\Abs{\wtl_{i,j}}}{1 + \NormII{\vw_{*,i}}^2} + \alpha} 
\end{align*}
In order for our algorithm  to correctly identify a terminal vertex in line $\ref{line:term_vertex_1}$,
we need to ensure that the lower bound on $r_i$ for a non-terminal vertex is strictly large than
the upperbound on $r_i$ for a terminal vertex. Therefore, we need to ensure that:
\begin{gather*}
\left( \frac{1}{\sigma^2} \right) \frac{\Abs{\wtl_{i,j}} - \alpha}{\frac{\Abs{\wtl_{i,j}}}{1 + \NormII{\vw_{*,i}}^2} + \alpha} 
 > \frac{2}{\sigma^2} 
\end{gather*}
Let $c_i$ be the number of children of the $i$-th node.
Then, using the fact that $\NormII{\vw_{*,i}}^2 \geq 9 c_i \alpha^2$i,  and the function on the
left hand side of the inequality above is an increasing function of $\Abs{\wtl_{i,j}}$, this further simplifies to 
\begin{align*}
\Abs{\wtl_{i,j}} - \alpha > 2 \left(\frac{\Abs{\wtl_{i,j}}}{1 + 9c_i \alpha^2} + \alpha \right) \\
\implies \Abs{\wtl_{i,j}} > \frac{3 \alpha}{1 - \frac{2}{1 + 9 c_i \alpha^2}}.
\end{align*}
Therefore, by Assumption \ref{ass:faithfulness} (ii), in line \ref{line:term_vertex_1} of Algorithm \ref{alg:main} 
we correctly identify a terminal vertex with probability at least $1 - 3 \delta$. 
Using an union bound over the $p-1$ iterations we conclude that, with probability 
at least $1 - 3 (p-1) \delta$, Algorithm \ref{alg:main} recovers a correct causal ordering of the nodes. 

Next in line \ref{line:theta}, the true coefficient vector satisfies: 
$\vth^* = \mSig_{z_i, \Sh_{z_i}}(\mSig_{\Sh_{z_i}, \Sh_{z_i}})^{-1} = 
\frac{\mbrOmg_{z_i, \Sh_{z_i}}}{\bar{\Omega}_{z_i, z_i}}$, where $\mbrOmg$
denotes the inverse covariance matrix over $X_{\Set{z_i} \union \Sh_{z_i}}$.
From the fact that, a node
is independent of its non-descendants given its parents, the non-zero entries of 
$\vth^*$ correctly identifies the parent set of $z_i$. 
Therefore, by RSAF (Assumption \ref{ass:faithfulness}), which states that the absolute value of
 the minimum non-zero entry in $\mbrOmg$ is at least $3 \alpha$,
we have that the support of the OLS estimate $\vhth$ in line
\ref{line:theta} correctly recovers the parent set for $z_i$ with high probability, i.e., 
$\Prob{\hPar{z_i}{} \neq \Par{z_i}{\G^*}} \leq 3 (p-1) \delta$.

Finally, from Lemma \ref{lemma:reg_coeffs} and another union bound over $p-1$ iterations of learning the parameters of the GBN, we get that $\Abs{\mB^* - \mhB}_{\infty} \leq \alpha$
with probability at least $1 - 6 (p-1) \delta$. Together with condition (i) of Assumption \ref{fig:faithfulness}, this implies $\widehat{\Es} = \Es^*$ with probability at least
$1 - 6 (p-1) \delta$. Setting $6 (p-1) \delta = \delta'$ for some 
$\delta' \in (0, 1)$ we prove our claim.
\end{proof}

\section{Additional Experiments}
\label{app:experiments}

\subsection{Our method vs PC algorithm on a non-faithful GBN}
We ran our method and the PC algorithm on the example given in Figure \ref{fig:rsaf_example}. We sampled $50000$ samples from the GBN to ensure that the CI tests used by the PC algorithm are accurate. The following figure shows, from left to right, the true graph, the graph learned by our algorithm (with edge weights rounded to two decimal places), and the graph recovered by the PC algorithm.
\begin{center}
\includegraphics[width=0.2\linewidth]{fig_rsaf}
\includegraphics[width=0.2\linewidth]{fig_rsaf}
\includegraphics[width=0.2\linewidth]{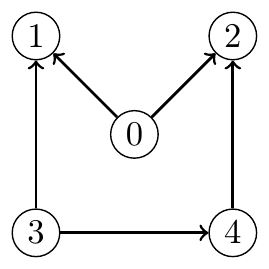}
\end{center}

\subsection{Unequal noise variance}
We set out to understand the performance of our algorithm when we relax the assumption of equal noise variance.
Clearly, in this case, we no longer have identifiability of the true DAG structure.
Therefore, we instead ask the following experimental question: ``What fraction of the true edges can we recover if we perturb the noise variance of the nodes slightly?'' For this experiment, we sampled GBNs as described in the previous paragraph. However, instead of setting the noise variance to be $0.8$ for all nodes, we set the noise variance for each node to be one of $\Set{1, 1 - \gamma, 1 + \gamma}$ with probability $\nicefrac{1}{3}$, where $\gamma$ is the noise parameter. From Figure \ref{fig:acc_recall} we note that in the regime
where the noise variance of the different nodes varies by $0.125$, i.e., between
$0.9375$ and $1.0625$, we still achieve close-to-perfect recovery.
\begin{figure}[htbp]
\begin{center}
\includegraphics[width=0.5\linewidth]{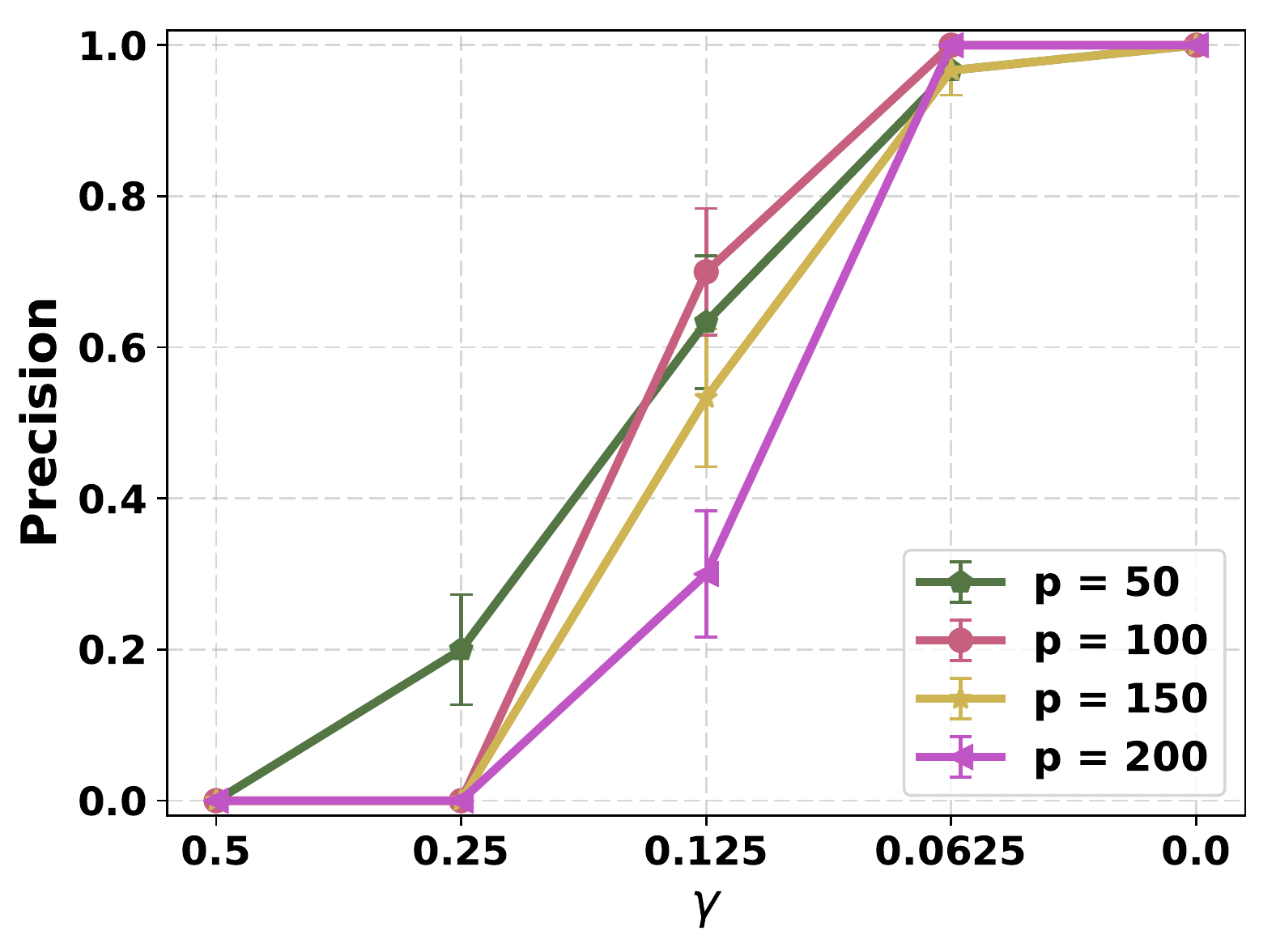}%
\includegraphics[width=0.5\linewidth]{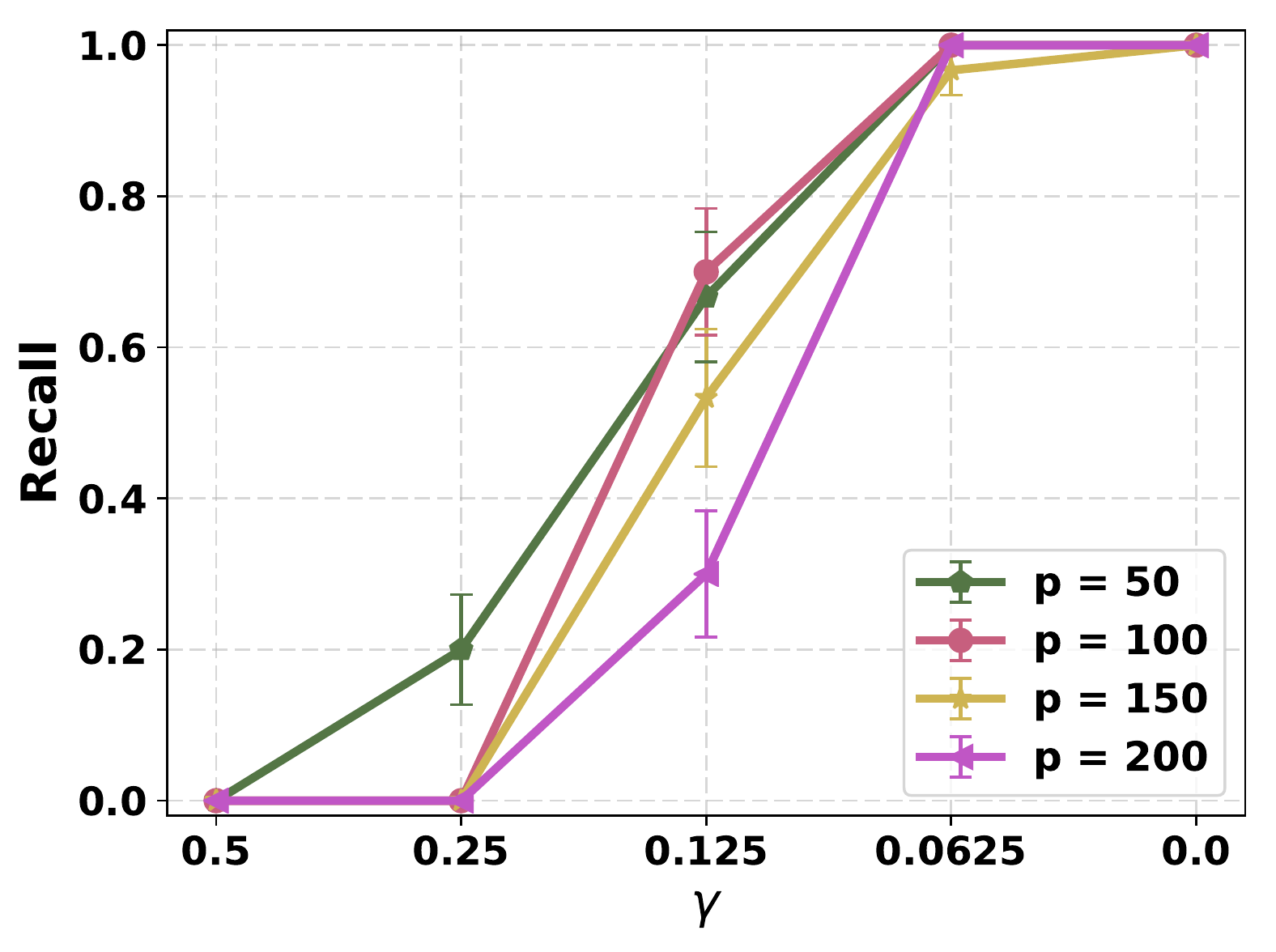}
\end{center}
\caption{Precision and Recall vs. noise parameter $\gamma$, where
the noise variance for each variable was set to one of
$\Set{1, 1 - \gamma, 1 + \gamma}$ with equal probability.
As $\gamma$ decreases, the accuracy and recall increases and we
achieve perfect recovery when $\gamma = 0$, i.e. when the variables have
equal noise variance. \label{fig:acc_recall}}
\end{figure}

\subsection{Experiments on real-world data}
We used gene expression data for $590$ subjects with breast invasive carcinoma from the \emph{cancer genome atlas} dataset.
The dataset is publicly available at \url{http://tcga-data.nci.nih.gov/tcga/}.
We used $187$ genes commonly regulated in cancer that were identified on independent datasets by~\cite{Lu07}.
The genes are the following:

\sloppy
{ \scriptsize
\vspace{0.05in}\noindent
ABCA8, ABHD6, ACLY, ADAM10, ADAM12, ADHFE1, AGXT2, ALDH6A1, 
ANK2, ANKS1B, ANP32E, AP1S1, APOL2, ARL4D, ARPC1B, AURKA, 
AYTL2, BAT2D1, BAX, BFAR, BID, BOLA2, BRP44L, C10orf116, 
C17orf27, C1orf58, C1orf96, C5orf4, C6orf60, C8orf76, CALU, CARD4, 
CASC5, CBX3, CCNB2, CCT5, CDC14B, CDCA7, CEP55, CHRDL1, 
CIDEA, CKLF, CLEC3B, CLU, CNIH4, DBR1, DDX39, DHRS4, 
DKFZp667G2110, DKFZp762E1312, DMD, DNMT1, DTL, DTX3L, E2F3, ECHDC2, 
ECHDC3, EFCBP1, EFHC2, EIF2AK1, EIF2C2, EIF2S2, Ells1, EPHX2, 
EPRS, ERBB4, FAM107A, FAM49B, FARP1, FBXO3, FBXO32, FEN1, 
FEZ1, FKBP10, FKBP11, FLJ11286, FLJ14668, FLJ20489, FLJ20701, FLJ21511, 
FMNL3, FMO4, FNDC3B, FOXP1, FTL, GEMIN6, GLT25D1, GNL2, 
GOLPH2, GPR172A, GSTM5, GULP1, HDGF, HIF3A, HLA-F, HLF, 
HNRPK, HNRPU, HPSE2, HSPE1, ILF3, IPO9, IQGAP3, K-ALPHA-1, 
KCNAB1, KDELC1, KDELR2, KDELR3, KIAA1217, KIAA1715, LDHD, LOC162073, 
LOC91689, LRRFIP2, LSM4, MAGI1, MORC2, MPPE1, MSRA, MTERFD1, 
NAP1L1, NCL, NDRG2, NME1, NONO, NOX4, NPM1, NR3C2, 
NRP2, NUSAP1, P53AIP1, PALM, PAQR8, PDIA6, PGK1, PINK1, 
PLEKHB2, PLIN, PLOD3, PPAP2B, PPIH, PPP2R1B, PRC1, PSMA4, 
PSMA7, PSMB2, PSMB4, PSMB8, PTP4A3, RBAK, RECK, RORA, 
RPN2, SCNM1, SEMA6D, SFXN1, SHANK2, SLAMF8, SLC24A3, SLC38A1, 
SNCA, SNRPB, SNX10, SORBS2, SPP1, STAT1, SYNGR1, TAP1, 
TAPBP, TCEAL2, TMEM4, TMEPAI, TNFSF13B, TNPO1, TRPM3, TTK, 
TTL, TUBAL3, UBA2, USP2, UTP18, WASF3, WHSC1, WISP1, 
XTP3TPA, ZBTB12, ZWILCH.
}
\fussy

After learning the DAG, we computed how many nodes are reachable from each of the $187$ nodes.
We found out that the gene CCNB2 reaches the greatest number of nodes among all genes ($163$ nodes).
Interestingly, this gene was independently found to be associated with an unfavorable outcome for breast-cancer patients in treatment~\cite{Shubbar13}.
As specifically mentioned by~\cite{Shubbar13} ``findings suggest that cytoplasmic CCNB2 may function as an oncogene and could serve as a potential biomarker of unfavorable prognosis over short-term follow-up in breast cancer''.

\subsection{Learning GBNs using marginal variance}
\label{fig:noise_var}
To ensure that the class of GBNs used in our synthetic experiments were non-trivial: meaning the marginal variance of the nodes did not give away the causal ordering, we tested another algorithm, which we will call the marginal-variance algorithm, to compute the DAG order by simply sorting the nodes according to their marginal variance. Figure \ref{fig:sort_variance} shows the probability of successful structure recovery across 30 randomly sampled GBNs, for the marginal-variance algorithm. We can observe that the marginal-variance algorithm fails to recover the DAG structure much more frequently as the number of variables grows. At $p = 200$, the algorithm fails to recover the true structure $50\%$ of the time. \vspace*{-15pt}
\begin{figure}[htbp]
\begin{center}
	\includegraphics[width=0.55\textwidth]{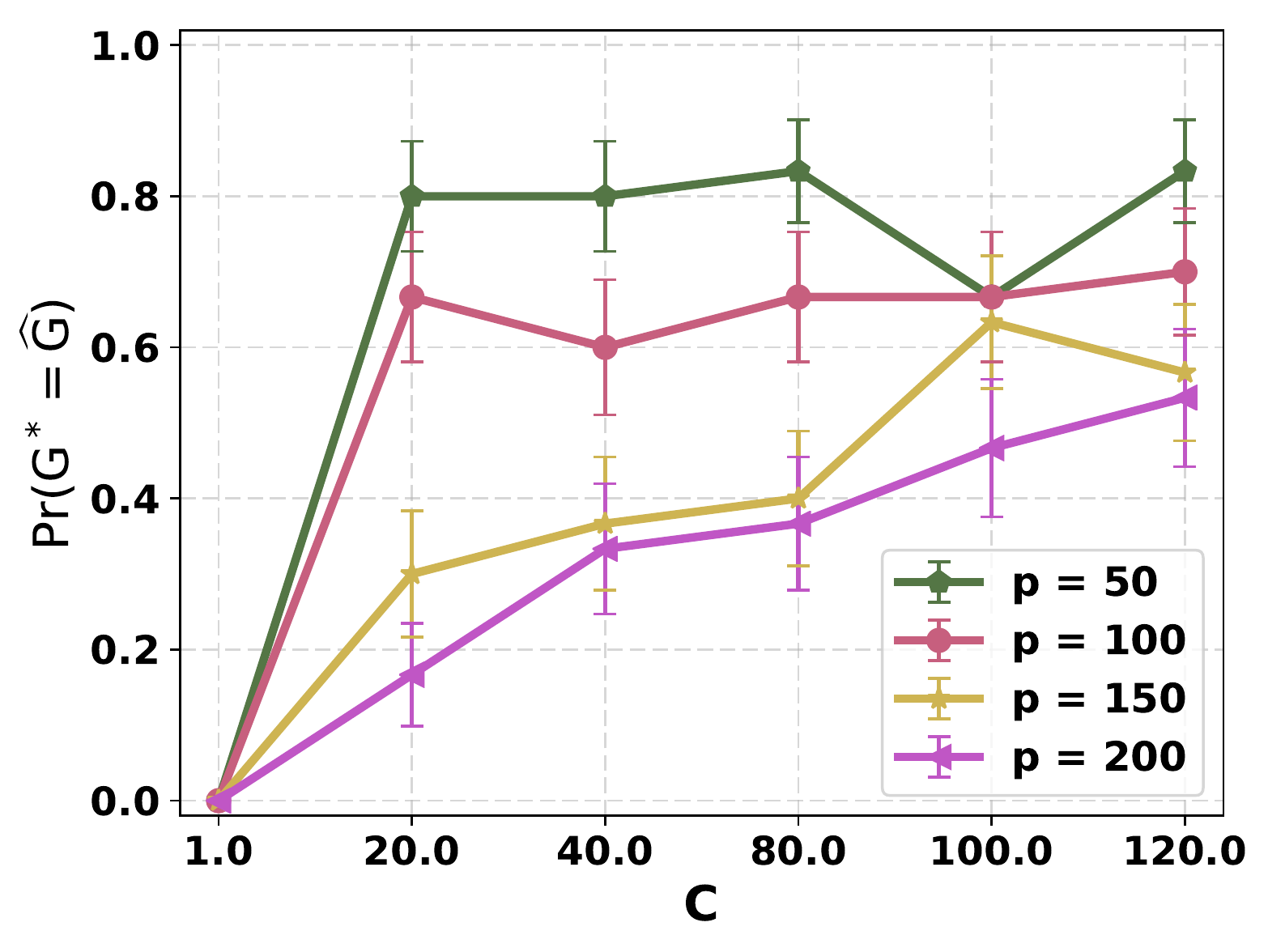}
\end{center}
\caption{Performance of the marginal-variance algorithm that uses sorting of the nodes by marginal variance to learn the DAG order.%
\label{fig:sort_variance}}
\end{figure}

\section{Discussion}
\label{app:appendix_discussion}
\subsection{Computational Complexity}
The computational complexity of our algorithm is dominated by the inverse covariance 
estimation step. As described in \cite{cai_constrained_2011}, the CLIME 
estimator of the inverse covariance matrix can be obtained by solving $p$ linear programs,
each with $2p$ inequality constraints in a $4p$-dimensional vector space. Each of these linear
programs can be solved in polynomial time by using interior point methods. Further,
state-of-the-art methods for inverse covariance estimation can potentially scale to a million
variables \cite{hsieh_big_2013}. After estimating
the inverse covariance matrix, our algorithm performs $(p - 1)$ OLS computations in 
(at-most) $\R^k$, to learn the DAG order and another $(p - 1)$ OLS computations to learn the
structure and parameters. This can be accomplished in $\BigO{p k^3}$ time
by directly inverting (at-most) $k \times k$ symmetric positive-definite matrices.
Thus, it is safe to conclude that our exact algorithm for learning equal noise-variance GBNs is highly scalable. 

\subsection{Using RESIT for learning linear Gaussian SEMs}
\begin{proposition}
\label{prop:resit_prop}
Let $(\G, \Pf(\Ws, \Ss))$ be a GBN 
and $\mX \in \R^{p}$ be a data sample drawn from $\Pf$. For
any variable $i$, let $\vth^*_i = \min_{\vth \in \R^{(p-1)}} 
	\frac{1}{2} \Exp{}{(X_{i} - \vth^T \mX_{\mi})^2}$, and let 
$R_i = X_i - (\vth^*_i)^T X_{\mi}$ be the $i$-th population residual.
Then, the residual $R_i$ is independent of $X_j$ for all $j \in \mi$, i.e.,
$\Cov{}{R_i, X_j} = 0$.
\end{proposition}
A consequence of the above proposition is that, RESIT, which identifies terminal
vertices, and subsequently the DAG order, by performing independence tests between
the residual $R_i$ and the covariates $X_{\mi}$, does not work even in
the population setting.

\begin{proof}[Proof of Proposition \ref{prop:resit_prop}]
Without loss of generality, let us write the joint distribution of $(X_i, X_{\mi})$
as follows:
\begin{align*}
\matrx{X_{\mi} \\ X_i} \sim \mathcal{N} 
	\left(\vect{0}, \matrx{\mA & \vb \\ \vb^T & c} \right).
\end{align*}
Then, from standard results for ordinary least squares, we have that
\begin{align*}
\vth^*_i = \argmin_{\vth \in \R^{p-1}} 
 \Exp{}{\frac{1}{2}\NormII{\mX_{*,i} - \mX_{*,\mi} \vth}^2} = \mA^{-1} \vb.
\end{align*}
Let $R_i = X_i - \vb^T \mA^{-1} X_{\mi}$. Since both $R_i$ and $X_{\mi}$ are mean 0, we get that:
$\Cov{}{R_i, X_{\mi}} = \Exp{}{R_i X_{\mi}^T} - \Exp{}{R_i} \Exp{}{X_{\mi}^T} 
	= \Exp{}{X_i X_{\mi}^T} - \Exp{}{\vb^T \mA^{-1} X_{\mi} X_{\mi}^T} 
	= \vb^T - \vb^T \mA^{-1} \mA = \vect{0}$
\end{proof}

\end{appendices}

\end{document}